\documentclass[conference]{IEEEtran}
\usepackage{times}

\let\labelindent\relax

\usepackage{amsmath,amsfonts,bm}
\usepackage{graphicx}

\def\1{\bm{1}}

\DeclareMathAlphabet{\mathsfit}{\encodingdefault}{\sfdefault}{m}{sl}
\SetMathAlphabet{\mathsfit}{bold}{\encodingdefault}{\sfdefault}{bx}{n}

\usepackage{xcolor}
\usepackage{xspace}

\newcommand{\eg}{e.g.,\xspace}

\definecolor{taskred}{HTML}{CC3838}
\definecolor{taskblue}{HTML}{4E7EC0}
\definecolor{taskorange}{HTML}{C07838}
\definecolor{taskpurple}{HTML}{6D4FC7}
\definecolor{taskdarkpurple}{HTML}{7B6FC4}

\definecolor{dsrl}{HTML}{7d87e8}
\definecolor{rlpd}{HTML}{459ba4}
\definecolor{spirl}{HTML}{ea51c7}
\definecolor{postbc}{HTML}{95923f}
\definecolor{tmrl}{HTML}{3eb076}
\definecolor{tmrl_cfg}{HTML}{3eb076}

\newcommand{\dsrl}{\textcolor{dsrl}{{{{DSRL}}}}}
\newcommand{\rlpd}{\textcolor{rlpd}{{{{RLPD}}}}}
\newcommand{\postbc}{\textcolor{postbc}{{{{PostBC}}}}}
\newcommand{\spirl}{\textcolor{spirl}{{{{SPiRL}}}}}
\newcommand{\tmrl}{\textcolor{tmrl}{{{{TMRL}}}}}
\newcommand{\tmrlcfg}{\textcolor{tmrl}{{{{TMRL-CFG}}}}}

\newcommand{\pedcond}{$p\!\left(a \mid\mspace{-6mu}\raisebox{-0.2\height}{\includegraphics[height=0.9em]{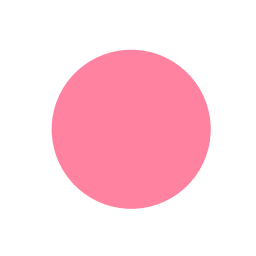}}\right)$}
\newcommand{\goal}{\textcolor[HTML]{ff83a1}{goal}}
\newcommand{\agent}{\textcolor[HTML]{efbe58}{agent}}

\usepackage[numbers]{natbib}
\usepackage[colorlinks=true, allcolors=blue]{hyperref}
\hypersetup{
  pdftitle={TMRL: Diffusion Timestep-Modulated Pretraining Enables Exploration for Efficient Policy Finetuning},
  pdfauthor={Matthew M. Hong, Jesse Zhang, Anusha Nagabandi, Abhishek Gupta},
  pdfsubject={Robot Learning},
  pdfkeywords={reinforcement learning, behavior cloning, diffusion policies, exploration}
}
\usepackage[nameinlink,capitalize]{cleveref}
\usepackage{subcaption}
\usepackage{wrapfig}
\usepackage{xcolor}
\usepackage{algorithm}
\usepackage{algorithmic}
\usepackage{amsmath, amssymb, amsthm}
\newtheorem{theorem}{Theorem}
\usepackage{url}
\usepackage{graphicx}
\usepackage{booktabs}
\usepackage{enumitem}
\usepackage{mathtools}

\newcommand{\method}{\textsc{TMRL}\xspace}
\newtheorem{corollary}{Corollary}[theorem]

\title{TMRL: Diffusion Timestep-Modulated Pretraining Enables Exploration for Efficient Policy Finetuning}

\begin{document}

\author{
Matthew M. Hong$^{1}$,
Jesse Zhang$^{1}$,
Anusha Nagabandi$^{2}$,
Abhishek Gupta$^{1}$\\[4pt]
{\small
$^{1}$University of Washington\quad
$^{2}$Amazon FAR\quad
}\\
}

\makeatletter
\let\@oldmaketitle\@maketitle
\renewcommand{\@maketitle}{\@oldmaketitle
    \includegraphics[width=\linewidth]{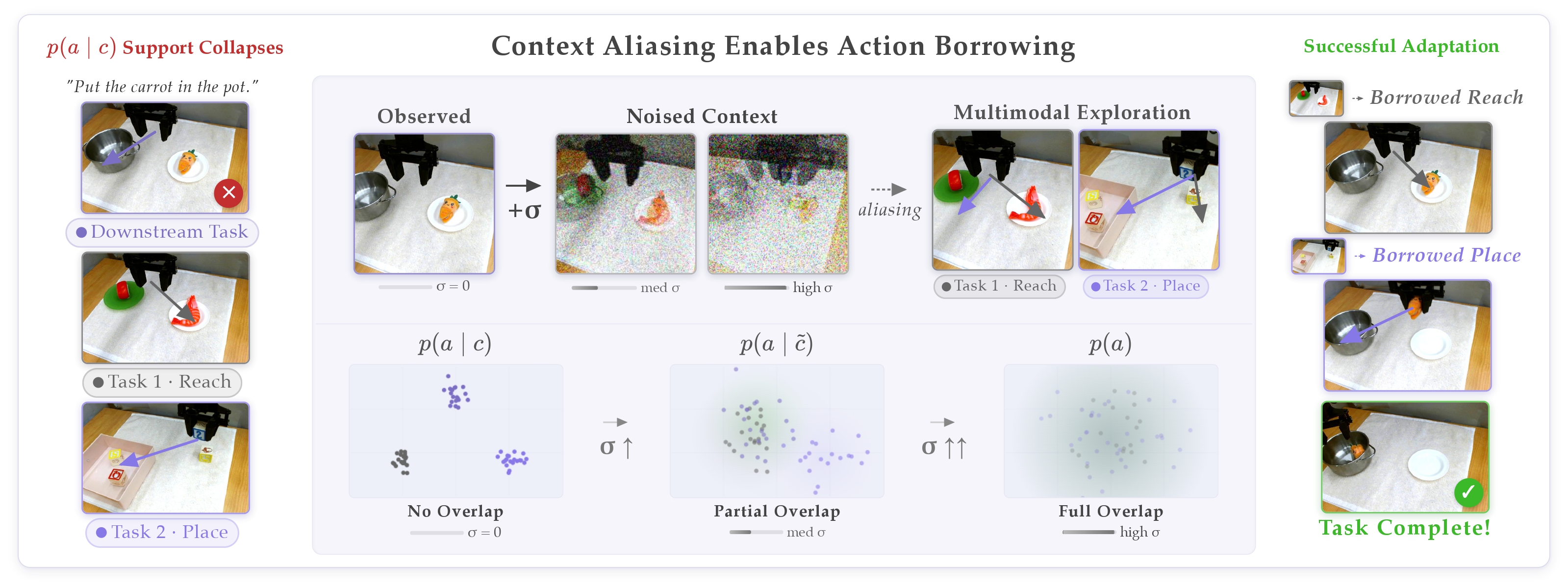}
       \refstepcounter{figure}\footnotesize{{Fig. 1:} TMRL bridges behavior cloning (BC) pre-training and RL fine-tuning by smoothing the conditioning of policy inputs (contexts). During pre-training, Context-Smoothed Pre-training (CSP) injects noise into contexts $c$, inducing a continuum from sharp imitation $p(a \mid c)$ to broader, marginal action distributions $p(a)$ via diffusion noise parameterized by $\sigma$ (bottom row). This smoothing causes nearby contexts to overlap in representation space, with similar contexts (e.g., \textcolor{taskdarkpurple}{\textbf{downstream task}}, \textcolor{gray}{\textbf{task~1}} in left column) merging at lower $\sigma$ than dissimilar ones (e.g., \textcolor{dsrl}{\textbf{task~2}}). During RL fine-tuning, TMRL \emph{learns} to dynamically \emph{modulate} conditioning strength, interpolating between context-conditioned and exploratory actions for improved exploration and adaptation.}
  \label{fig:teaser} \medskip \vspace{-10pt}}
\makeatother
\maketitle

\begin{abstract}
Fine-tuning pre-trained robot policies with reinforcement learning (RL) often inherits the bottlenecks introduced by pre-training with behavioral cloning (BC), which produces narrow action distributions that lack the coverage necessary for downstream exploration. We present a unified framework that enables the exploration necessary to enable efficient robot-policy finetuning by bridging BC pre-training and RL fine-tuning. Our pre-training method, Context-Smoothed Pre-training (CSP), injects forward-diffusion noise into policy inputs, creating a continuum between precise imitation and broad action coverage. We then fine-tune pre-trained policies via Timestep-Modulated Reinforcement Learning (TMRL), which trains the agent to dynamically adjust this conditioning during fine-tuning by modulating the diffusion timestep, granting explicit control over exploration. Integrating seamlessly with arbitrary policy inputs, e.g., states, 3D point clouds, or image-based VLA policies, we show that TMRL improves RL fine-tuning sample efficiency. Notably, TMRL enables successful real-world fine-tuning on complex manipulation tasks in under one hour.
Videos and code available at \url{https://weirdlabuw.github.io/tmrl/}.
\end{abstract}  

\renewcommand\thefigure{\arabic{figure}}
\setcounter{figure}{1}

\IEEEpeerreviewmaketitle

\section{Introduction}
A dominant paradigm for training real-world robotic policies is to first pre-train on large-scale demonstration datasets via imitation learning, and then fine-tune the resulting policy with reinforcement learning (RL) in deployment environments~\citep{wagenmaker2025posteriorbehavioralcloningpretraining, wagenmaker2025steering, intelligence2025pi06vla, zhang2023bootstrap, zhang2024sprint, luo2025precise, zhang2024extract, hu2025flare, yin2025sgft, yang2023robofume}. 
This RL fine-tuning stage is critical for improving task precision, throughput and robustness~\citep{lei2026rl100performantroboticmanipulation, luo2025precise, li2025grrlgoingdexterousprecise}, yet remains bottlenecked by sample efficiency due to the cost of real-world interaction. While prior work in robotics primarily focuses on improving the RL algorithms themselves, comparatively little attention has been paid to ensuring that pre-trained policies provide effective initializations for downstream RL. 
In this work, we propose a combined robot policy pre-training and RL fine-tuning framework that enables exploration for sample-efficient adaptation.

Standard behavior cloning (BC) trains a policy to imitate demonstrator actions~\citep{Pomerleau-1989-15721, Bain1995AFF}. 
When demonstrations densely cover a context $c$ (e.g., observation and task instruction), BC accurately models the conditional action distribution $p(a \mid c)$~\citep{florence2021implicit, chi2023diffusionpolicy}. 
However, in sparsely covered regions and under distribution shift, BC overfits to observed data: the conditional support collapses and BC can assign near-zero probability to optimal actions. 
As a result, online rollouts during downstream adaptation yield little reward, causing reinforcement learning to fail to receive the learning signal needed to improve. 

Building on this insight, recent work has examined the relationship between pre-training via supervised learning and RL fine-tuning~\citep{springer2025overtrained, chen2025rethinking, zeng2025pretrainingindicatorsreliablypredict, wagenmaker2025posteriorbehavioralcloningpretraining, chen2026the}, primarily in language domains. 
In robotics specifically, \citet{wagenmaker2025posteriorbehavioralcloningpretraining} shows that achieving high \emph{action coverage}—the ability to sample actions a demonstrator would take—is \emph{necessary} for efficient RL fine-tuning. 
Thus, they inject Gaussian action noise during training to increase coverage.
However, there is a fundamental trade-off: random actions guarantee coverage but lack task proficiency, while BC policies can achieve high proficiency but have limited out-of-distribution coverage. A naive solution is to collect more data, but doing so for every new setting can be costly. Instead, we ask: how can we pre-train policies that \emph{adaptively} broaden their action distribution during RL?

Our key insight is to train policies that \emph{interpolate} between the learned conditional $p(a \mid c)$ and marginal $p(a)$ action distributions.
While $p(a \mid c)$ offers precise behavior in familiar contexts, $p(a)$ provides broad coverage needed for exploration. Interpolating between these extremes enables balancing exploitation and exploration based on the context's novelty.

We instantiate this idea through \textit{Context-Smoothed Policies (CSPs)}, a pre-training framework that expands the support of BC policies by injecting noise into context inputs via a \emph{forward diffusion process}. Introducing low levels of diffusion noise can alias similar contexts $c$~\citep{hu2025stemob}. Aliasing the input contexts leads to learning a broader conditional action distribution, since actions are ``borrowed" across nearby contexts. As the noise increases, the policy smoothly transitions from a narrow conditional distribution $p(a \mid c)$ toward a broad marginal distribution $p(a)$. At maximum noise, $c$ contains no information and the policy recovers $p(a)$~\citep{zhu2025uwm}, ensuring full training distribution coverage. We pre-train CSPs across \emph{all} noise levels, enabling smooth interpolation between these extremes. Therefore, the tradeoff between action coverage and conditioning fidelity can be selected at inference time.

Finally, having established a pre-training framework that exposes this tradeoff, the central challenge becomes how best to harness this flexible initialization during downstream RL finetuning. During RL fine-tuning, the optimal amount of context conditioning necessary---and therefore the balance between broad exploration and precise execution---can vary dramatically even within a single trajectory. To fully exploit our pre-trained CSPs, we introduce \emph{Timestep-Modulated Reinforcement Learning} (\method) to dynamically adjust the diffusion timestep during RL. This mechanism provides an RL agent with an explicit control variable that modulates conditioning strength, enabling it to interpolate between conditional and marginal behaviors in the pre-trained policy and thereby explore more effectively. 
Ideally, a policy \emph{learns} when to rely on precise imitation and when to broaden its action support. In practice, TMRL is simple and easy to implement, enabling steering/fine-tuning of diverse policies, from state-input diffusion policies to vision-language-action (VLA) models.

We evaluate our proposed methodology across a range of simulated and real-world robotic tasks. First, we show that CSPs alone substantially improve action coverage and achieve stronger zero-shot performance on unseen tasks over standard BC and prior pre-training approaches. \method then converts this coverage advantage into significantly better RL sample efficiency on manipulation and navigation tasks in simulation, outperforming state-of-the-art steering methods even when the base policy contains sparse behavioral coverage. We further show that context-smoothing extends naturally to VLA and 3D-input policies, where noising VLA embeddings and point clouds, respectively, enables broader exploration across image-input and dexterous manipulation tasks. Finally, we demonstrate that our approach scales to the real world, enabling rapid RL of manipulation behaviors within \textit{an hour} of experiment time while the baseline achieves near-zero success. 
\section{Related Work}
Existing research on RL fine-tuning in robotics spans a wide spectrum of topics. In this section, we review these paradigms and situate both context-smoothed policies and \method\ within them. Specifically, we differentiate our approach in (i) the design of pre-training objectives for action coverage, (ii) the explicit, dynamic tuning of exploration optimism during post-training, and (iii) the role of restricted information in inducing structured state aliasing.

\textbf{Pre-training for RL Fine-tuning.}
Standard practice in robot learning involves pre-training policies on large-scale demonstration datasets via offline RL or BC, followed by RL fine-tuning~\citep{wagenmaker2025posteriorbehavioralcloningpretraining, wagenmaker2025steering, intelligence2025pi06vla, zhang2023bootstrap, zhang2024sprint, luo2025precise, zhang2024extract, hu2025flare, yin2025sgft, yang2023robofume, lei2026rl100performantroboticmanipulation, luo2025precise, li2025grrlgoingdexterousprecise}. However, these imitation objectives learn the narrow conditional action distribution, $p(a \mid c)$. In low-density or out-of-distribution contexts, this conditional distribution overcommits to observed data, collapsing its support and assigning near-zero probability to potentially optimal actions.

Recent findings in language modeling suggest that standard imitation losses are a poor predictor of downstream fine-tuning performance~\citep{springer2025overtrained, chen2025rethinking, zeng2025pretrainingindicatorsreliablypredict, chen2026the}. Instead, PostBC~\citep{wagenmaker2025posteriorbehavioralcloningpretraining} demonstrates that successful RL fine-tuning of robot policies requires sufficient \emph{action coverage} to enable RL to sample and amplify optimal behaviors. 
PostBC specifically fits the posterior distribution of the demonstrator's behavior using an ensemble of single-action policies, then injects Gaussian noise into the actions at low-density states to broaden the pre-training distribution by using a covariance estimate from the ensemble. 
While mathematically sufficient for coverage, injecting unstructured noise directly into the action space often leads to ``dithering'' and execution-level incoherence, as the noise is typically uncorrelated across timesteps.

Instead, we introduce \emph{context-smoothed policies}. Rather than perturbing actions, we inject sampled diffusion noise into the \emph{input} space of a BC policy during pre-training. By shifting the perturbation to the context, we enable an RL agent to borrow coherent action sequences from nearby contexts, the marginal, or a mix of the two via diffusion timestep interpolation, guaranteeing exploratory coverage without the dithering inherent to unstructured action-space noise.

\textbf{RL Fine-tuning and Policy Steering.}
Once pre-trained, policies are often improved via RL fine-tuning. Some approaches apply entropy regularization or action-noise strategies in offline or fine-tuning settings~\citep{wu2019behavior, kumar2020CQL, pmlr-v202-ball23a}, or directly fine-tune pre-trained diffusion policies~\citep{ren2025diffusion}, but still rely on standard BC policies, which may overfit and perform poorly in unseen deployment settings. More structured methods perform \emph{steering} over a frozen pre-trained policy~\citep{pertsch2020spirl, singh2021parrot, ajay2021opal, pertsch2021skild, zhang2024extract, wagenmaker2025steering, dong2025expo}, training a high-level RL policy to effectively modulate the base policy's output. However, in these settings, steering is strictly limited to the support of the base policy's conditional action distribution $p(a \mid c)$, so if $p(a\mid c)$ collapses in unseen contexts, as standard BC policies routinely do, the optimal action is unreachable regardless of how well the RL policy is trained. \method\ addresses this collapse by allowing the agent to dynamically tune its reliance on context conditioning, interpolating between $p(a \mid c)$ and the marginal distribution $p(a)$, enabling adaptive exploration beyond the base policy while remaining grounded in offline data.

\textbf{Generalization via Restricted Information.}
Finally, context-smoothed policies align with work showing that restricting input information can improve policy generalization and robustness. Prior work demonstrated that intentionally limiting observations can lead to better robustness in control, imitation learning, and image generation~\citep{tomar2023ignorance, goyal2023infobottransferexplorationinformation, song2025historyguided, chen2025diffusion, zhang2025peek, zisselman2025blindfolded, hu2025stemob, mirjalili2026augmented, li2025controlvla}. For example, Stem-Ob~\citep{hu2025stemob} trains policies with diffusion-noised visual conditioning to improve zero-shot generalization. 
Most similar to ours, Classifier-Free Guidance~\citep{ho2022classifierfreediffusionguidance} can also improve generalization by training both a conditional and a marginal policy, where the conditional policy \emph{guides} the marginal policy, but our ablations in \Cref{sec:exp:ablations} demonstrate that CSPs are a more effective pre-training procedure for RL fine-tuning on new tasks.

A closely related idea is that of \emph{randomized smoothing}, used in stochastic optimization~\citep{duchi2012randomizedsmoothing}, for smoothing discontinuous gradients~\citep{suh2022bundled}, or to increase classifier robustness~\citep{cohen2019certified}, in which random noise is applied to smooth a vector. Closest to our analysis, \citet{block2023provable} smooth the \emph{observations} of BC policies, proving that the smoothed policy is total variation continuous and leveraging this property to derive distribution-matching guarantees for generative BC; our theory establishes an analogous TV-Lipschitz property (\Cref{thm:tv_lipschitz_overlap}), but deploys smoothing as a \emph{controllable} coverage-expansion mechanism for RL fine-tuning rather than for imitation stability. For \method, we train a \emph{context-smoothed} policy that sees conditioning (context) inputs \emph{smoothed} by a noising diffusion forward process. This context-smoothed policy serves as a steerable base policy, which \method\ trains an RL policy over, enabling explicit and tunable control over exploration optimism.

\section{The Need for Action Distribution Interpolation}
\label{sec:pedogogical_example}
We begin with a simple experiment that illustrates why we need a \emph{pre-training} procedure that can interpolate between the marginal action distribution $p(a)$ and a conditional distribution $p(a \mid c)$ when encountering unseen tasks.
In a simple 2D point-maze from the simulated OGBench benchmark~\citep{ogbench_park2025}, we train two policies to approximate $p(a)$ and $p(a\mid c)$  where $c$ is the goal position.
These policies are trained on the \texttt{pointmaze-large-navigate} from OGBench~\citep{ogbench_park2025} and then evaluated on a larger, out-of-distribution maze with an unseen goal position.
In \Cref{fig:pedagogical}, we visualize samples from the two distributions and the corresponding paths taken at various starting positions, gradually approaching the goal.
The agent is in \textcolor[HTML]{efbe58}{yellow} and the goal position in \textcolor[HTML]{ff83a1}{pink}. \Cref{fig:pedagogical} demonstrates that when the agent is far away, \pedcond\ is overfit to the training distribution and the sampled action sequences lead it \emph{further away} from the goal.
Meanwhile, $p(a)$'s action coverage is very broad, but does place more probability mass on actions that take the agent closer to the goal.
As the agent gets closer to the goal, $p(a \mid c)$ becomes more useful for reaching the goal.

\begin{figure}[htb]
    \centering
    \includegraphics[width=\linewidth]{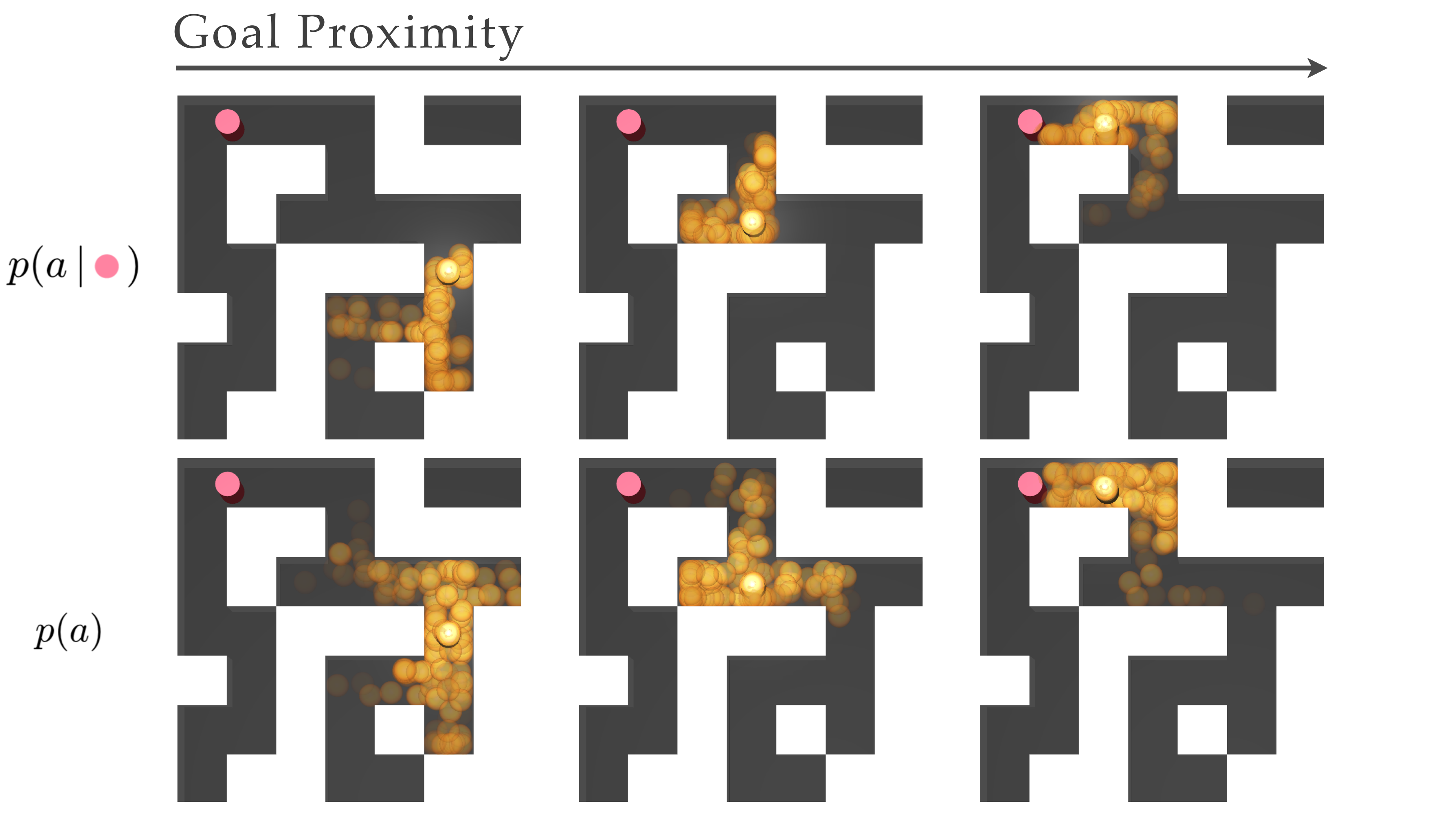}
    \caption{For an unseen \goal\ position, the conditional action distribution \protect\pedcond\ remains narrow, lacking sufficient support for the \agent\ to reach the \goal. In contrast, the marginal $p(a)$ provides the required coverage, though with too much action variance near the goal. In this setting, an optimal policy should more closely match $p(a)$ at the beginning of the trajectory and \protect\pedcond\ at the end. \emph{Context-smoothed} policy pre-training enables TMRL to \emph{interpolate} between both distributions for more effective RL fine-tuning.
    }
    \label{fig:pedagogical}
\end{figure}

\textbf{Takeaways for RL.} In general cases where $p(a \mid c)$ is overfit, $p(a)$ represents an action distribution that is more likely to cover the optimal action distribution necessary to solve the task. 
However, a pre-training recipe that na\"{i}vely trains both policies directly is insufficient for two reasons: (1) it is unclear how a downstream RL algorithm can decide when to pick each distribution to solve a given downstream task; and (2) while $p(a)$ can provide greater coverage of the optimal action, it represents an overly broad distribution, which can make downstream learning slow and unsafe.

To resolve both problems, we propose pre-training a policy via \textbf{context-smoothing} to interpolate between $p(a \mid c)$ and $p(a)$. Then, at inference time, we fine-tune with \textbf{TMRL} to learn \textit{when} and \textit{how} to move along this interpolation, dynamically adjusting coverage to explore effectively while simultaneously learning the new task. At convergence, TMRL interpolates between $p(a)$ and $p(a \mid c)$ to maximize optimal action coverage (2) and also to determine when to lean more on either distribution to solve the task (1).

\begin{figure*}[t]
    \centering
    \includegraphics[width=0.95\linewidth]{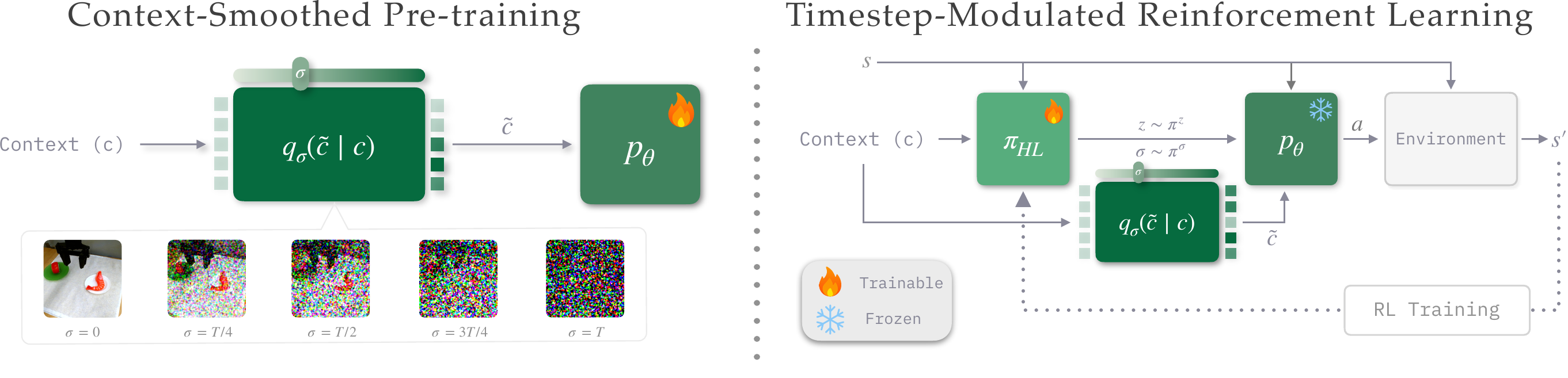}
    \caption{\textbf{Timestep-modulated exploration via context smoothing.} (\textit{Left}) During pre-training, a steerable policy $p_\theta$ is trained across all noise levels $\sigma$ by corrupting the context $c$ via the kernel $q_\sigma(\tilde{c}\mid c)$, producing a policy that can be queried at any conditioning strength during inference. (\textit{Right}) During RL fine-tuning, \method exposes $p_\theta$ with a \emph{context-noise dial} $\sigma$ as an explicit control variable for the high-level policy $\pi_{\text{HL}}$.}
    \label{fig:method}
\end{figure*}

\section{Timestep Modulated RL on Context-Smoothed Policies}
\label{sec:method}

\subsection{Problem setting}
\label{sec:method:setting}

\textbf{Pre-training.} We assume access to an offline dataset of near-optimal pre-training trajectories
$\mathcal{D}=\{\tau_j\}$, where each trajectory $\tau_j=\{(c_i,a_i)\}_{i=1}^{T_j}$ contains context--action pairs.
In this paper, we use context to abstractly define a variety of inputs to the policy, such as states, RGB images, point clouds, language instructions, and other task descriptors. 
More generally than maximum likelihood, we learn a policy via a supervised objective with a user-specified loss $\ell$:
\begin{equation}
\min_\theta\ 
\mathbb{E}_{(c_t,a_t)\sim\mathcal{D}}
\Big[
\ell \big(\theta;\ c_t,a_t\big)
\Big].
\label{eq:generic_supervised}
\end{equation}
This formulation subsumes imitation learning as a special case: choosing $\ell(\theta;c,a)=-\log p_\theta(a\mid c)$ recovers maximum likelihood (behavior cloning). Other choices of $\ell$ yield different training algorithms and model classes, \eg denoising/score-matching objectives for score-based or diffusion policies. 

In this work, we will focus on a class of generative control policies (GCPs) \citep{chi2023diffusionpolicy, zhang2025affordancebasedrobotmanipulationflow, pan2025adonoisingdispellingmyths}, where policies are represented by generative models such as diffusion/flow models and are trained to maximize a lower bound of the action likelihood.\footnote{For simplicity, we will stick to the $(a, c)$ notation, although this is naturally extended to $H$-step action chunks $a^{t:t+H}$ and context histories $c_{t-H:t}$.} Note that our pre-training and fine-tuning procedures are general and do not particularly rely on GCPs. 

\textbf{RL Fine-tuning.} 
Given this pre-trained policy, our goal is to develop an efficient RL fine-tuning algorithm in novel contexts. Of particular interest in this work is the class of RL algorithms for ``steering" pre-trained policies~\citep{wagenmaker2025steering, wagenmaker2025posteriorbehavioralcloningpretraining, su2026rfs}, although our formulation is more broadly applicable. 
In GCPs, the inference procedure allows for steerability via a latent variable $z$, giving $p(a\mid c, z)$. For flow/diffusion policies, this corresponds to the initial noise used during denoising~\citep{wagenmaker2025steering}. 
Thus, the RL agent represents a high-level policy $\pi_{\mathrm{HL}}(z\mid c)$ that selects the latent variable $z$ to maximizing downstream reward $r$ sum discounted by $\gamma$ in conjunction with $p(a|c, z)$:
\begin{align}
\label{eq:hl_objective}
    \max_{\pi_{\mathrm{HL}}}\ &\mathbb{E}\Big[\sum_{t\ge 0}\gamma^t r_t \Big], \\
\quad
\mathrm{s.t.} \; \; &z_t\sim\pi_{\mathrm{HL}}(\cdot\mid c_t),
\quad
a_t^{1:H}\sim p_\theta(\cdot\mid c_t, z_t). \nonumber
\end{align}

\textbf{Action Coverage.} A fundamental bottleneck in RL fine-tuning is distribution shift: in novel contexts, the pre-trained conditional distribution $p_\theta(\cdot\mid c,z)$ often overcommits, resulting in a collapsed action distribution. \citet{wagenmaker2025posteriorbehavioralcloningpretraining} formalizes this issue through the concept of \emph{demonstrator action coverage}. Specifically, a pre-trained policy achieves coverage with parameter $\kappa > 0$ if, $\forall (c, a)$:
\begin{equation}
p_\theta(a \mid c) \ge \kappa \cdot p^\beta(a \mid c),
\label{eq:coverage_def}
\end{equation} where $p^\beta$ represents the true demonstrator policy.
This coverage is a strict prerequisite for successful RL fine-tuning. Since fine-tuning algorithms rely on online rollouts to update behaviors, a collapsed initialization (where $\kappa \to 0$ for unseen optimal actions) will never sample the transitions required to discover meaningful reward signal. 
Because standard BC fails to guarantee meaningful coverage in low-density regions, it often yields initializations incapable of downstream learning.
Our method resolves this by instantiating a new class of RL fine-tuning algorithms based on ``context-smoothed pre-training'', which allows the RL agent to adaptively modulate the amount of action coverage needed for different contexts $c$.

\subsection{Context-smoothing}
\label{sec:method:cs}

Intuitively, context-smoothed policies make a simple change to standard policy pre-training: they inject noise into the context $c$ while being pre-trained on the offline dataset $\mathcal{D}$. We show that doing so enables policies to adaptively expand their action coverage during evaluation and RL fine-tuning.

To define this formally - let $q_\sigma(\tilde c\mid c)$ be a corruption kernel that injects noise into the context, with noise scale $\sigma\ge 0$. We define the \emph{context-smoothed} steerable policy as the mixture
\begin{align*}
p_{\theta,\sigma}(a\mid c,z)
&~\coloneqq~
\mathbb{E}_{\tilde c\sim q_\sigma(\cdot\mid c)}
\Big[p_\theta(a\mid \tilde c, z)\Big] \nonumber \\
&~=~
\int p_\theta(a\mid \tilde c,z)\,q_\sigma(\tilde c\mid c)\,d\tilde c.
\end{align*}
We refer to $p_{\theta,\sigma}$ as a \emph{context-smoothed policy}. Given a context $c$, action inference with a context-smoothed policy amounts to corrupting the context with $\tilde c \sim q_\sigma(\tilde c\mid c)$, and then sampling actions with the corrupted context from $p_\theta(a\mid \tilde c,z)$. 

\noindent\textbf{Intuition behind context smoothed policies:} For $\sigma\downarrow 0$, $q_\sigma(\tilde c\mid c)$ concentrates near $c$, so $p_{\theta,\sigma}(\cdot\mid c,z)$ approaches the original conditional controller.
As $\sigma$ increases, the corrupted context $\tilde c$ becomes less informative, and the induced action distribution becomes a broader mixture over behaviors associated with \emph{nearby / aliased} contexts.
This controlled interpolation between the conditional distribution $p(a|c)$ and the marginal $p(a)$ creates \emph{structured action coverage expansion}: instead of random action noise, the policy can borrow coherent action chunks from related contexts present in the dataset.

\noindent\paragraph{Implementation of corruption kernel $q_\sigma(\tilde c\mid c)$}
We define the corruption kernel $q_\sigma$ via an iterative forward-noising process over contexts, much like the forward process of a diffusion model. Let $c_0 \equiv c$ and choose a variance schedule $\{\beta_t\}_{t=1}^{T_c}$ with $\alpha_t\!=\!1-\beta_t$ and
$\bar\alpha_t \coloneqq \prod_{i=1}^t \alpha_i$.
The forward process implies the closed-form marginal:
$q(c_t \mid c_0)
=
\mathcal{N}\!\big(\sqrt{\bar\alpha_t}\,c_0,\ (1-\bar\alpha_t)I\big).$
We take the corruption kernel at noise level $t_c$ to be
\begin{equation}
q_\sigma(\tilde c \mid c)
\equiv
q_{t_c}(\tilde c \mid c)
=
\mathcal{N}\!\big(\sqrt{\bar\alpha_{t_c}}\,c,\ (1-\bar\alpha_{t_c})I\big).
\label{eq:corruption_kernel_noise}
\end{equation}
Since the forward corruption process is entirely determined by how many time-steps it is run forward, the noise-level $\sigma$ for $q_\sigma(\tilde c\mid c)$ can be parameterized entirely by controlling the \textbf{diffusion timestep} $t_c\in\{0,\dots,T_c\}$, hence the moniker timestep modulated RL. We use $\sigma$ to refer to corruption kernels $q_\sigma$ generally and $t_c$ for our specific instantiation with diffusion noise.

\subsection{Pre-training a context-smoothed policy}
\label{sec:method:training}

Training a context-smoothed policy requires learning action predictors conditional on corrupted context - $p_\theta(a\mid \tilde c, z)$. Towards realizability, we train a policy class that can be queried at arbitrary context-noise levels by providing $\sigma$ (or $t_c$) as an explicit input to the policy - $p_\theta(a\mid \tilde c, z, \sigma)$. Concretely, we introduce a policy - $p_\theta(a\mid \tilde c, z, \sigma)$, and train it so that, for each training pair $(c,a) \sim \mathcal{D}$, we maximize action likelihood while noising the context with a small modification to \Cref{eq:generic_supervised}:
\begin{equation}
\min_\theta\ 
\mathbb{E}_{(c,a)\sim\mathcal{D}}
\ \mathbb{E}_{\substack{\sigma\sim\mathcal{S} \\ \tilde c\sim q_\sigma(\cdot\mid c)}}
\
\Big[
\ell\big(\theta;\ a, \tilde c, \sigma\big)
\Big],
\label{eq:train_generic}
\end{equation}
where $\mathcal{S}$ is a distribution over noise levels and $\ell$ is any supervised learning objective suitable for the chosen steerable policy class.
In our specific instantiation using GCPs as $p_\theta$, $\ell$ is the denoising score-matching or flow-matching loss \citep{song2019generativemodeling, ho2020denoising, song2021denoising} and $\sigma$ corresponds to the diffusion noise timestep used to corrupt $\tilde{c}$, $t_c$. Notably, this is \emph{not} a much more complicated procedure than standard imitation learning; it merely introduces \emph{controlled context smoothing} during pre-training, similar to data augmentation, and conditioning $p_\theta$ on the amount of smoothing $\sigma$. As we show next, this context-smoothed pre-trained policy can result in a massive improvement for efficient RL fine-tuning. See \Cref{alg:cs_train} for pseudocode.

\begin{algorithm}[t]
\caption{Context-smoothed pre-training.}
\label{alg:cs_train}
\small
\begin{algorithmic}[1]
\STATE \textbf{Input:} dataset $\mathcal{D}=\{(c,a^{1:H})\}$, corruption kernel $q_\sigma$, noise-level sampler $\sigma\sim\mathcal{S}$
\STATE \textbf{Initialize:} parameters $\theta$ of $p_\theta(a^{1:H}\mid c,z,\sigma)$ (and any latent prior over $z$ if used)
\WHILE{not converged}
    \STATE Sample $(c,a^{1:H})\sim\mathcal{D}$
    \STATE Sample $\sigma\sim\mathcal{S}$ (diffusion timestep $t_c$ in our instantiation)
    \STATE Sample $\tilde c\sim q_\sigma(\cdot\mid c)$ (\cref{eq:corruption_kernel_noise})
    \STATE Update $\theta$ to minimize $\ell(\theta;\ a^{1:H},\tilde c,\sigma)$ (\cref{eq:train_generic})
\ENDWHILE
\end{algorithmic}
\end{algorithm}

\begin{figure*}[!h]
    \centering
    \includegraphics[width=0.9\linewidth]{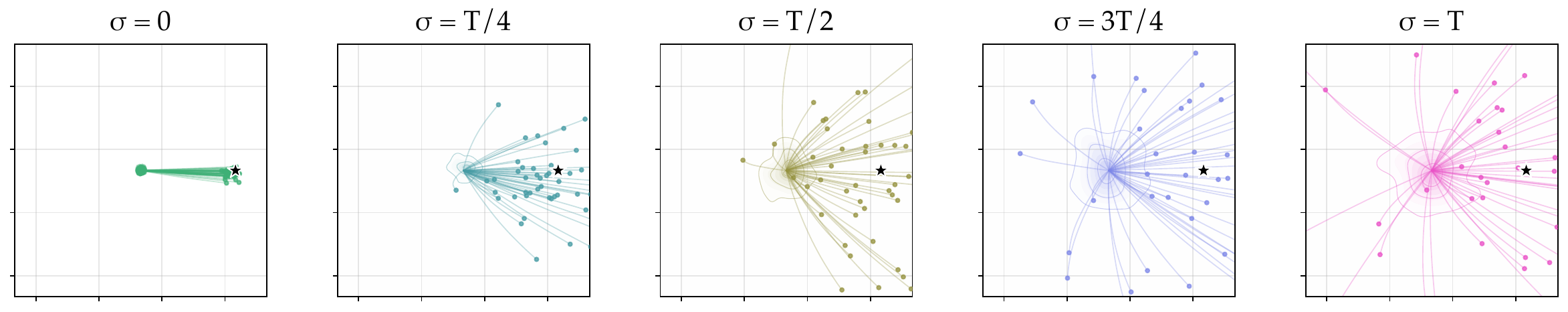}
    \caption{\textbf{Empirical validation of context smoothing.}
    We train a context-smoothed diffusion policy $p(x, y \mid \tilde{\theta}, \sigma)$ with \Cref{eq:train_generic} to produce 2D points on a unit circle conditioned on $c = \theta$, where $\sigma \in {0, 1, \ldots, T}$ denotes diffusion scheduler timesteps. Each panel shows $p(x, y \mid \tilde{\theta} = 0, \sigma)$ at a fixed level of context noise $\sigma$; the shaded region and contours visualize the spread of the conditioning distribution over context inputs at each $\sigma$. As $\sigma \rightarrow 0$, conditioning sharpens — the contours collapse to a point — and outputs concentrate around the unit circle at $\theta = 0$; as $\sigma$ increases, the contours widen, reflecting an induced action distribution that becomes broader and overlaps more across nearby contexts.}
    \label{fig:borrowing-exp}
\end{figure*}

\subsection{RL fine-tuning with timestep-modulated context smoothing}
\label{sec:method:rl}

Compared to standard pre-training, training a context-smoothed policy $p_{\theta}(a^{1:H}\mid c,z, \sigma)$ naturally provides an avenue to expand the exploration space during RL fine-tuning. For example, in \emph{steering algorithms}~\citep{wagenmaker2025steering}, which RL fine-tune by training a high-level RL policy to select the latent noise $z$ to initialize action selection from a GCP (see \Cref{sec:method:setting} for a refresher), the high-level policy can now select the steering vector $z$ \emph{and} the context-smoothing noise level $\sigma$ applied to $c$:
\begin{equation*}
\resizebox{0.91\columnwidth}{!}{$
(z_t,\sigma_t)\sim \pi_{\mathrm{HL}}(\cdot\mid s_t),
\hspace{0.1cm}
\tilde c_t \sim q_{\sigma_t}(\cdot \mid c(s_t)),
\hspace{0.1cm}
a_t \sim p_{\theta}(\cdot \mid \tilde c_t, z_t, \sigma_t).
$}
\label{eq:hl_sigma}
\end{equation*}

With a frozen context smoothed base policy, we optimize $\pi_{\mathrm{HL}}(z,\sigma\mid s)$ using any RL method (\eg off-policy actor-critic~\citep{haarnoja2018sac}) to maximize the expected return in \Cref{eq:hl_objective}. Intuitively, $\sigma$ provides an exploration--exploitation, action coverage (\cref{eq:coverage_def}) dial: large $\sigma$ aliases contexts (broader borrowing, better exploration, more coverage), while small $\sigma$ sharpens conditioning (better exploitation, lower action coverage) once progress is found. See \Cref{alg:tmrl} for pseudocode.

\begin{algorithm}[t]
\caption{Timestep-Modulated RL (TMRL), simplified.}
\label{alg:tmrl}
\small
\begin{algorithmic}[1]
\STATE \textbf{Input:} pretrained context-smoothed controller $p_\theta(\cdot\mid c,z,\sigma)$; initialize replay buffer $\mathcal{B}$
\STATE Initialize high-level actor $\pi_{\mathrm{HL}}(z,\sigma\mid s)$ and critic(s)
\FOR{each training iteration}
    \STATE Observe $s_t$ and set $c_t\leftarrow c(s_t)$
    \STATE Sample $(z_t,\sigma_t)\sim \pi_{\mathrm{HL}}(\cdot\mid s_t)$, $\tilde{c}_t \leftarrow$ $c_t$ smoothed with $\sigma_t$
    \STATE Sample and execute action chunk $a_t^{1:H}\sim p_\theta(\cdot\mid \tilde{c}_t,z_t,\sigma_t)$
    \STATE Store transitions in $\mathcal{B}$ and update critics with off-policy RL
    \STATE Update $\pi_{\mathrm{HL}}$ using critic gradient (details in App. \Cref{sec:appendix:impl_details})
\ENDFOR
\end{algorithmic}
\end{algorithm}

\section{Theoretical Analysis: Benefits of Context Smoothing}
\label{sec:method:analysis}

In this section, we provide theoretical evidence and empirical analysis to better understand the gains from timestep-modulated, context-smoothed policies for RL fine-tuning. 

\paragraph{Empirical Analysis}

We first empirically validate that increasing $\sigma$ (or $t_c$) yields smooth, meaningful behavioral changes. We show in \Cref{fig:borrowing-exp} that as corruption noise increases, the distribution of actions at a given context broadens, thereby increasing overlap with nearby contexts. Taken to the extreme, this shows interpolation between the conditional $p(a \mid c)$ and marginal $p(a)$ action distribution.

\paragraph{Theoretical Analysis}
Our theoretical analysis formalizes how, under Gaussian context-smoothing, smoothing mitigates the coverage collapse of standard BC.
Informally, we make two claims: \textbf{(i)} increasing context noise increases action distribution overlap across different contexts, and \textbf{(ii)} closer contexts overlap more than farther contexts. 
Referring back to \Cref{eq:coverage_def}, these claims together show that at an individual context $c$, if there is a different context $c'$ which increases demonstrator action coverage at $c$, then increasing context corruption $\sigma$ increases overlap with the higher action coverage distribution, $p(\cdot \mid c')$. 

\paragraph{Setup (Gaussian context smoothing)}
\label{sec:method:gaussian_context_smoothing_setup}
For a conditional action distribution $p(\cdot\mid c)$ with $c\in\mathbb{R}^d$, consider its Gaussian context-smoothed version: 
\begin{equation}
    p_\sigma(\cdot\mid c) \ \coloneqq\ \mathbb{E}_{w\sim\mathcal{N}(0,I_d)}\big[p(\cdot\mid c+\sigma w)\big].
    \label{eq:gaussian_smoothing_def}
\end{equation}
Let $\mathrm{TV}(\cdot,\cdot)$ denote total variation distance, and define overlap $\mathrm{Ov}(P,Q)\coloneqq 1-\mathrm{TV}(P,Q)$. 
This smoothed-policy construction mirrors randomized smoothing, which certifies classifier robustness by Gaussian-noising inputs~\citep{cohen2019certified}, and the observation-smoothed policies of \citet{block2023provable}, who show that smoothing endows any policy with \emph{total variation continuity} (TVC) and use this property to derive imitation guarantees for generative BC. Whereas these works employ smoothing for robustness and distributional stability, we use it to expand and adaptively control action coverage for RL fine-tuning.
We make the following statements about $p_\sigma(\cdot\mid c)$.

\begin{theorem}[Smoothing increases overlap and makes the policy Lipschitz in context]
\label{thm:tv_lipschitz_overlap}
For any two contexts $c,c'\in\mathbb{R}^d$,
\begin{equation}
    \mathrm{TV}\!\big(p_\sigma(\cdot\mid c),p_\sigma(\cdot\mid c')\big) ~\le~ \frac{\mathbb{E}\|w\|}{\sigma}\,\|c-c'\|.
    \label{eq:tv_lipschitz}
\end{equation}
Equivalently,
\begin{equation}
\mathrm{Ov}\!\big(p_\sigma(\cdot\mid c),p_\sigma(\cdot\mid c')\big) ~\ge~ 1-\frac{\mathbb{E}\|w\|}{\sigma}\,\|c-c'\|. 
\label{eq:ov_lipschitz}
\end{equation}
\end{theorem}

\noindent\underline{\emph{Interpretation.}} \Cref{thm:tv_lipschitz_overlap} formalizes two central components of maximizing action coverage for RL fine-tuning:

\begin{itemize}

\item \textbf{More noising $\Rightarrow$ broader coverage.} For fixed $(c,c')$, increasing $\sigma$ decreases the upper bound in \Cref{eq:tv_lipschitz} and increases the lower bound in \Cref{eq:ov_lipschitz}. For example, if $c'$ is a state from the demonstration data and $c$ is a novel, out-of-distribution state, increasing $\sigma$ forces the policy at $c$ to overlap with the demonstrator's behavior at $c'$, thereby \emph{increasing demonstrator action coverage}.

\item \textbf{Closer contexts $\Rightarrow$ targeted, coherent coverage.} For fixed $\sigma$, the bound scales linearly with $\|c-c'\|$: nearby contexts yield larger guaranteed overlap than distant contexts. This ensures that the \emph{coverage expansion is not simply unstructured noise, but structured aliasing}: the policy borrows action chunks from related contexts.
\end{itemize}

See proof in Appendix \Cref{sec:appendix:proof}. The next corollary now uses \Cref{thm:tv_lipschitz_overlap} to demonstrate that if two contexts exhibit overlap \emph{before smoothing}, they will exhibit even greater overlap \emph{after smoothing}:


\begin{corollary}[Sufficiently noisy smoothing ensures increased coverage relative to the base, unsmoothed policy]
\label{cor:overlap_threshold}
Fix $c,c'$ and let $\Delta=c-c'$. Assume the pair $(c,c')$ satisfies a (local) identifiability lower bound:
\begin{equation}
    \mathrm{TV}\!\big(p(\cdot\mid c),p(\cdot\mid c')\big)\ \ge\ m\|\Delta\|
    \label{eq:identifiability}
\end{equation}
for some $m>0$.
If
\begin{equation}
    \sigma\ \ge\ \frac{\mathbb{E}\|w\|}{m},
    \label{eq:sigma_threshold}
\end{equation}
then smoothing reduces distinguishability (TV) and hence increases overlap:
\begin{equation}
    \mathrm{Ov}\!\big(p_\sigma(\cdot\mid c),p_\sigma(\cdot\mid c')\big) ~\ge~ \mathrm{Ov}\!\big(p(\cdot\mid c),p(\cdot\mid c')\big).
    \label{eq:overlap_increase}
\end{equation}
\end{corollary}

\noindent\underline{\emph{Interpretation.}} 
Standard BC can create disjoint, narrow action supports for highly identifiable contexts (identifiability \Cref{eq:identifiability}). Once $\sigma$ exceeds the threshold in \Cref{eq:sigma_threshold}, the smoothing-induced aliasing dominates in \Cref{eq:overlap_increase}. \textbf{This demonstrates that smoothing increases action coverage relative to a base BC policy}, mitigating the risk of zero-probability optimal actions during RL fine-tuning.
A complete proof appears in Appendix \Cref{sec:appendix:proof_corollary}.


\noindent\textbf{Takeaway for RL Fine-tuning.} 
We've shown that context smoothing provably increases overlap relative to the original conditional policy: for any context $c$, if $\exists \, c'$ with better demonstrator action coverage, this increased overlap directly translates into a higher coverage parameter $\kappa$ for $c$ in \Cref{eq:coverage_def}. 
These results allow the RL policy to adaptively change coverage for each $c$ by treating $\sigma$ as an adaptive coverage dial: large $\sigma$ guarantees broad action coverage by pulling the distribution toward the marginal $p(a)$, while small $\sigma$ preserves precise conditioning for exploitation.

Unlike unstructured noise, which expands coverage at the cost of execution coherence, context-smoothed policies expand coverage using coherent actions from nearby contexts.

\noindent\textbf{Summary.} In summary, we propose Timestep-Modulated Reinforcement Learning (TMRL) over \emph{context-smoothed policies}, reframing policy pre-training for RL post-training from \textit{unstructured} action noise into a \textit{structured} process of context-smoothing.
We first train a generative policy $p_\theta(a \mid \tilde{c}, z, \sigma)$ by injecting controlled noise into the context $c$ via a forward diffusion process. This creates a policy that can be queried at any noise level $\sigma \in [0, T_c]$, where $\sigma$ determines the conditioning noise and guarantees coverage over the prior data. Then, during RL fine-tuning, a high-level policy $\pi_{\text{HL}}(z, \sigma \mid s)$ learns to dynamically treat $\sigma$ as a coverage dial. By adaptively increasing $\sigma$, the agent aliases the current context with nearby contexts in the dataset, ``borrowing'' coherent behaviors to explore securely beyond the base policy's behavior distribution. Put together, our work shows the merit of rethinking the role of pre-training for finetuning, especially the responsiveness of the policy to noisy input representations. 
\begin{figure*}[!t]
    \centering
    \includegraphics[width=\textwidth]{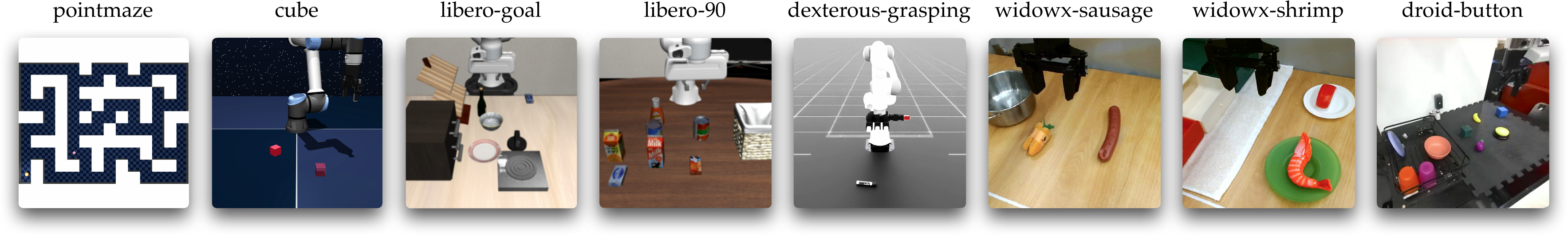}
    \caption{\textbf{Evaluation.} We evaluate \tmrl\ across 8 tasks spanning navigation and manipulation in simulation and the real world.}
    \label{fig:tasks}
\end{figure*}
\section{Experiments}

\begin{figure*}[ht]
    \centering
    \includegraphics[width=\linewidth]{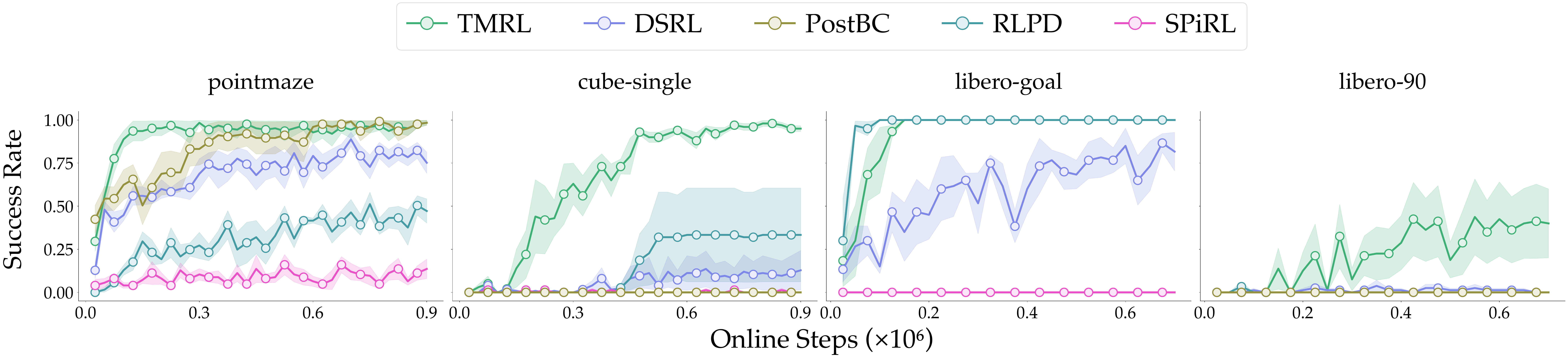}
    \caption{\textbf{RL Success Rates for simulation tasks.} \tmrl\ attains near 100\% success rate in both OGBench tasks, outperforming the best baselines by 14\% in \texttt{pointmaze-giant} and 200\% in \texttt{cube-single} at final performance. In \texttt{libero-goal}, \tmrl\ and \rlpd~\citep{pmlr-v202-ball23a} 
    both reach 100\% success. However, for the longer-horizon \texttt{libero-90} task, only \tmrl\ explores sufficiently to achieve non-trivial success rates.}
    \label{fig:sim-results}
\end{figure*}

Our experiments study the following research questions:
\begin{enumerate}[label=\textbf{(Q\arabic*)}  ]
    \item \label{q1} Does context smoothing produce better action coverage over other pre-training approaches? 
    \item \label{q2} Does \method\ effectively RL fine-tune policies trained across a variety of policy conditioning variables?
    \item \label{q3} Does \method\ enable real-world RL on VLAs?
    \item \label{q4} How does context-smoothing and \method\ compare to alternative coverage expansion mechanisms? How does \method\ learn to control context smoothing over time?
\end{enumerate}

\subsection{\ref{q1} Comparing Pre-Training Action Coverage}
\label{sec:experiments:pretraining}

We first compare the effect of various pre-training procedures on action coverage by measuring success rates on unseen tasks. We compare a context-smoothed pre-training (CSP) policy against standard BC, i.e., training $p(a \mid c)$ with \cref{eq:generic_supervised}, and against PostBC~\citep{wagenmaker2025posteriorbehavioralcloningpretraining}.
Overall, \textbf{CSP outperforms both baselines in zero-shot success rate of unseen tasks}.

We demonstrate this result on two tasks from OGBench~\citep{ogbench_park2025}: navigation (\texttt{pointmaze-giant}) and manipulation (\texttt{cube-single}). 
For our navigation task, we train policies on the \texttt{pointmaze-large-navigate} dataset and evaluate on the larger \texttt{pointmaze-giant} environment. The downstream environment has a larger state space, necessitating a broader action distribution. 
For our manipulation task, \texttt{cube-single}, we use a filtered \texttt{cube-single-play} dataset and evaluate adaptation for initial positions beyond this training dataset. See \Cref{sec:appendix:sim_exps} for further environment details.

In \Cref{fig:success-at-k}, we plot ``success @ $K$,'' measuring the fraction of out-of-distribution initial states for which at least one of the $K$ base policy rollouts succeeds.
This success rate represents a direct empirical measurement of \emph{demonstrator action coverage} from \cref{eq:coverage_def}.
On both OGBench tasks, CSP achieves higher success rates at every $K$ with a fixed smoothing $\sigma$, with the gap most pronounced on \texttt{cube} where BC and PostBC have zero success at all $K$. Next, we show that CSP's broader action coverage translates to better RL performance with TMRL.
\begin{figure}[!htb]
    \centering
    \includegraphics[width=\linewidth]{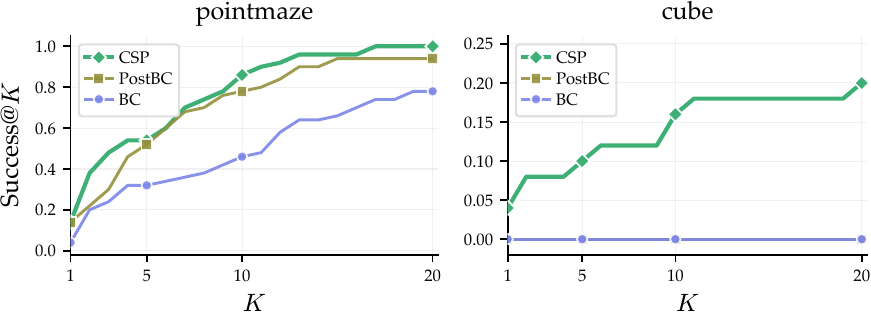}
    \caption{\textbf{CSP unlocks better action coverage before RL fine-tuning.} We measure the Success@$K$ for context-smoothed pre-training against standard BC and PostBC. CSP achieves greater success@$K$ across all $K$ on both tasks.}
    \label{fig:success-at-k}
\end{figure}

\subsection{RL Fine-tuning Baselines}
We compare \tmrl\ against prior approaches for RL with prior data or for combining pre-training and RL fine-tuning across multiple evaluation tasks, as shown in \Cref{fig:tasks}:
\begin{itemize}[leftmargin=10pt]
    \item \rlpd\ \citep{pmlr-v202-ball23a}: off-policy RL that learns online while incorporating offline data as an additional buffer. It relies on Gaussian action noise for expanding action coverage.
    \item \spirl\ \citep{pertsch2020spirl}: a hierarchical RL algorithm that trains an RL policy over pre-trained skills learned from offline data. It uses skill sampling for expanding coverage.
    \item \dsrl\ \citep{wagenmaker2025steering}: a steering algorithm that first trains a diffusion policy over action sequences with standard BC, then performs off-policy RL over the policy's noise space.
    \item \postbc\ \citep{wagenmaker2025posteriorbehavioralcloningpretraining}: pre-trains with additive Gaussian action noise to expand action coverage and then fine-tunes with DSRL.
\end{itemize}

\subsection{\ref{q2} RL Fine-tuning Across Varied Policy Conditioning}
\label{sec:experiments:sim_exps}
Now, we examine how \tmrl\ enables effective RL fine-tuning on difficult, unseen tasks. We consider downstream tasks that are out-of-distribution for the base policy, i.e., the training data contains no demonstrations for these downstream tasks. 
All results are means and standard deviations over 5 seeds.

\noindent\textbf{State-Based Conditioning.} As seen in the two left plots in \Cref{fig:sim-results}, we observe clear differences in how methods handle out-of-distribution generalization for the aforementioned OG bench domains from \Cref{sec:experiments:pretraining}. \rlpd\ achieves limited success, as it explores purely through additive Gaussian noise.

In contrast, methods that exploit action priors from the pre-training data perform better; \dsrl\ reaches around 90\% success rates in \texttt{pointmaze-giant}; however, it achieves near-0 success rates on the higher-dimensional action space task \texttt{cube-single} as it lacks any mechanism to go beyond the coverage of the base policy. \postbc\ performs very similarly to \dsrl---while it expands action coverage during pre-training, the coverage comes from single-step Gaussian noise, causing some dithering behavior, and, unlike \tmrl, cannot be adaptively controlled during RL fine-tuning. \tmrl\ consistently outperforms these baselines, achieving an overall 101\% improvement over the best-performing baseline across both tasks by \emph{learning} to systematically expand coverage of context-smoothed policies. Moreover, while several baselines exhibit high variance across seeds, \tmrl\ is notably more stable while still performing well. 

\noindent\textbf{Image-Based Tasks and VLM Embedding Conditioning:}  
We next consider a set of image-based tasks with a large pre-trained VLA policy, $\pi_0$ \citep{black2024pi0visionlanguageactionflowmodel}, which uses VLM embeddings to condition a flow-based action expert. We evaluate on the LIBERO benchmark \citep{liu2023libero}. To enable \tmrl\, we fine-tune a context-smoothed $\pi_0$ model on the \texttt{Libero}-\{\texttt{Spatial}, \texttt{Object}, \texttt{Goal}, \texttt{10}\} datasets by adding noise to the VLM embeddings context $c$ before they are input to the action expert head (see \Cref{fig:cs-vla}). We then evaluate adaptation on two unseen tasks: a position-perturbed version of \texttt{Libero-Goal} task \citep{zhou2025liberopro} and a task from the \texttt{Libero-90} task suite. We provide additional environment details in \Cref{sec:appendix:sim_exps}. 

\begin{figure}[h]
    \centering
    \includegraphics[width=\linewidth]{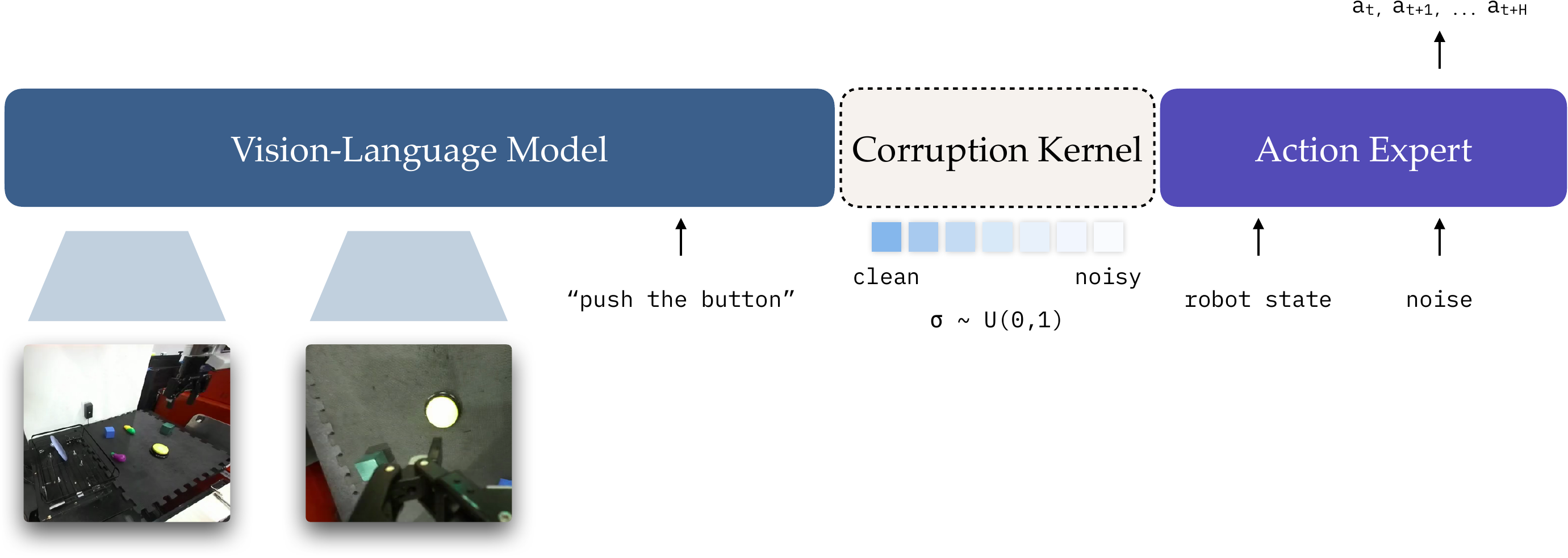}
    \caption{Context Smoothed Pre-training for VLAs: we propose applying the corruption kernel (\cref{eq:corruption_kernel_noise}) to VLM embeddings input to the action expert.}
    \label{fig:cs-vla}
\end{figure}

\Cref{fig:sim-results} shows that \tmrl\ leverages action chunks from different tasks in the dataset to solve the \texttt{libero-goal} task more efficiently than \dsrl. \rlpd\ also learns the task because it is very short-horizon. 
In \texttt{libero-90}, we observe that only \tmrl\ learns the task; \tmrl\ leverages action sequences from other tasks, meanwhile \dsrl\ overfits to picking up the same object each time, resulting in an overly narrow action distribution that is insufficient to solve the task. 

\noindent\textbf{Pointcloud Conditioning:} 
Next, we look at a simulated dexterous manipulation grasping setting with a LEAP hand~\citep{shaw2023leaphand} on a Franka in IsaacLab~\citep{mittal2025isaaclab}. 
We are interested in evaluating \tmrl's RL fine-tuning ability on a \emph{pointcloud}-input policy; we examine whether context-smoothing on pointclouds enables \tmrl\ to adapt to grasping new objects. We pre-train context-smoothed policies on three can-shaped objects and evaluate on a new marker-shaped object (as visualized in \Cref{fig:appendix:dex_objects}). 

\begin{wrapfigure}[10]{r}{0.49\linewidth}
    \centering
    \vspace{-0.6cm}
    \includegraphics[width=\linewidth]{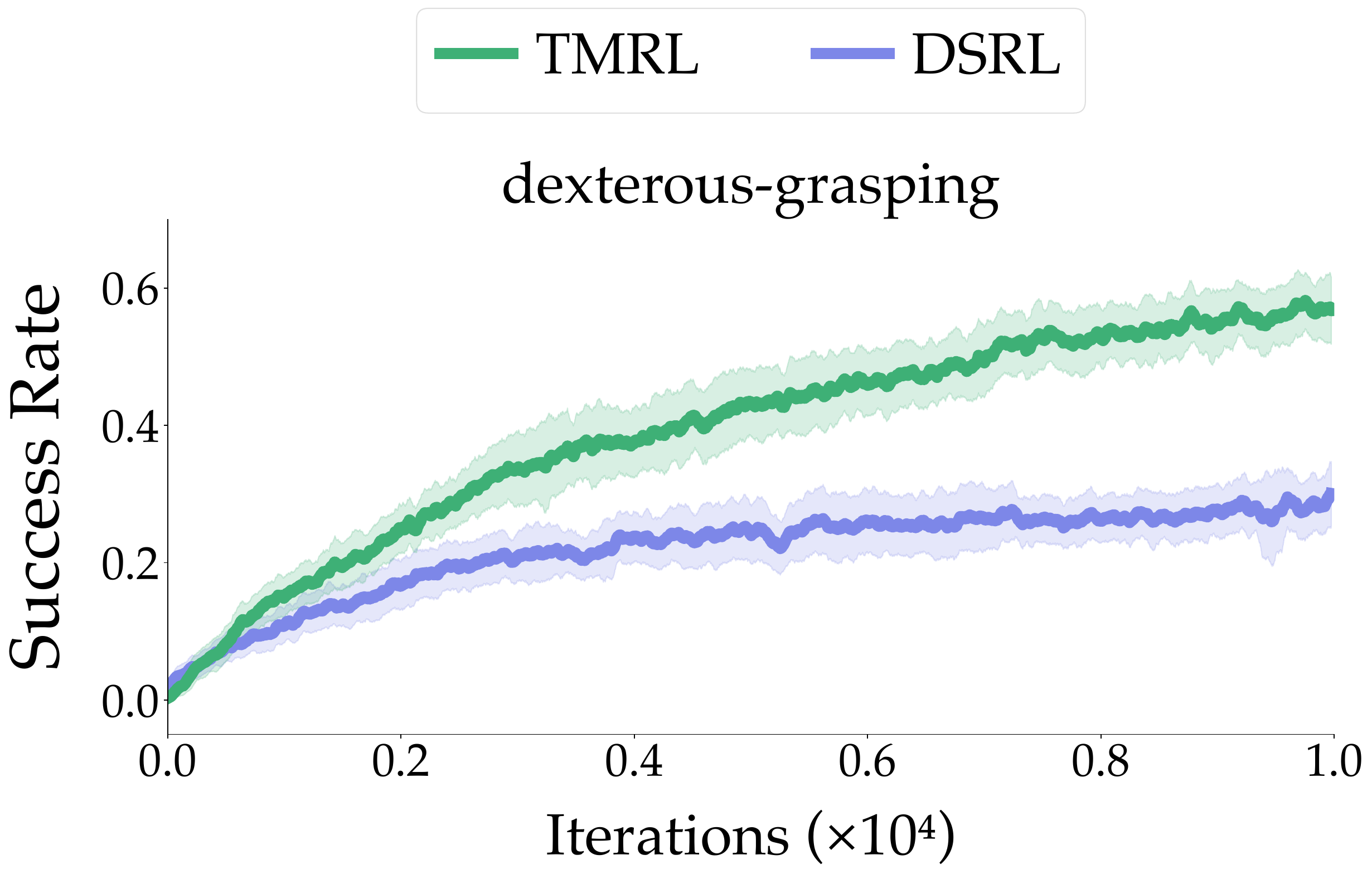}
    \caption{\tmrl\ enables efficient steering of 3D-input policies for dexterous grasping.}
    \label{fig:dex_curve}
\end{wrapfigure}
While grasp strategies do not immediately transfer, we find that noised pointclouds enable \tmrl\ to \emph{share grasping strategies} across different objects, enabling broader exploration and, consequently, faster learning and $2.5\times$ higher final success rates than \dsrl\ in \Cref{fig:dex_curve}.

\subsection{\ref{q3}: Does TMRL enable real-world RL on VLAs?}
\label{sec:experiments:real_exps}
Next, we demonstrate that TMRL enables effective adaptation of VLA policies in the real world. We evaluate on two platforms: a WidowX 250 6-DoF robot arm using the BridgeData-v2 \citep{pmlr-v229-walke23a} setup and dataset, and a Franka Panda 7-DoF robot arm using the DROID \citep{khazatsky2024droid} dataset. For each, we pre-train a context-smoothed policy by fine-tuning $\pi_0$~\cite{black2024pi0visionlanguageactionflowmodel} on the respective dataset using VLM embeddings as context $c$ and the context noising objective in \Cref{eq:train_generic}. Further implementation details are provided in \Cref{sec:appendix:real_exps}.

We evaluate \tmrl\ against \dsrl\ on three tasks, \texttt{sausage-in-pot}, \texttt{shrimp-in-white-drawer}, and \texttt{press-button}, with evaluation curves shown in \Cref{fig:real}. While the pre-trained policy is unable to solve the task successfully (often reaching for the wrong object), fine-tuning with \tmrl\ improves success rates to near-perfect levels. This is in stark contrast to \dsrl~\cite{wagenmaker2025steering}, which performs poorly because it cannot perform tasks beyond the coverage of the base policy. 
We omit a \postbc\ comparison here because the simulation performance in \Cref{fig:sim-results} is similar to that of \dsrl\ and it requires 2 additional steps: 1-step Gaussian ensemble policy pre-training and data labeling, then $\pi_0$ VLA re-training.

\begin{figure}[h]
    \centering
    \includegraphics[width=\linewidth]{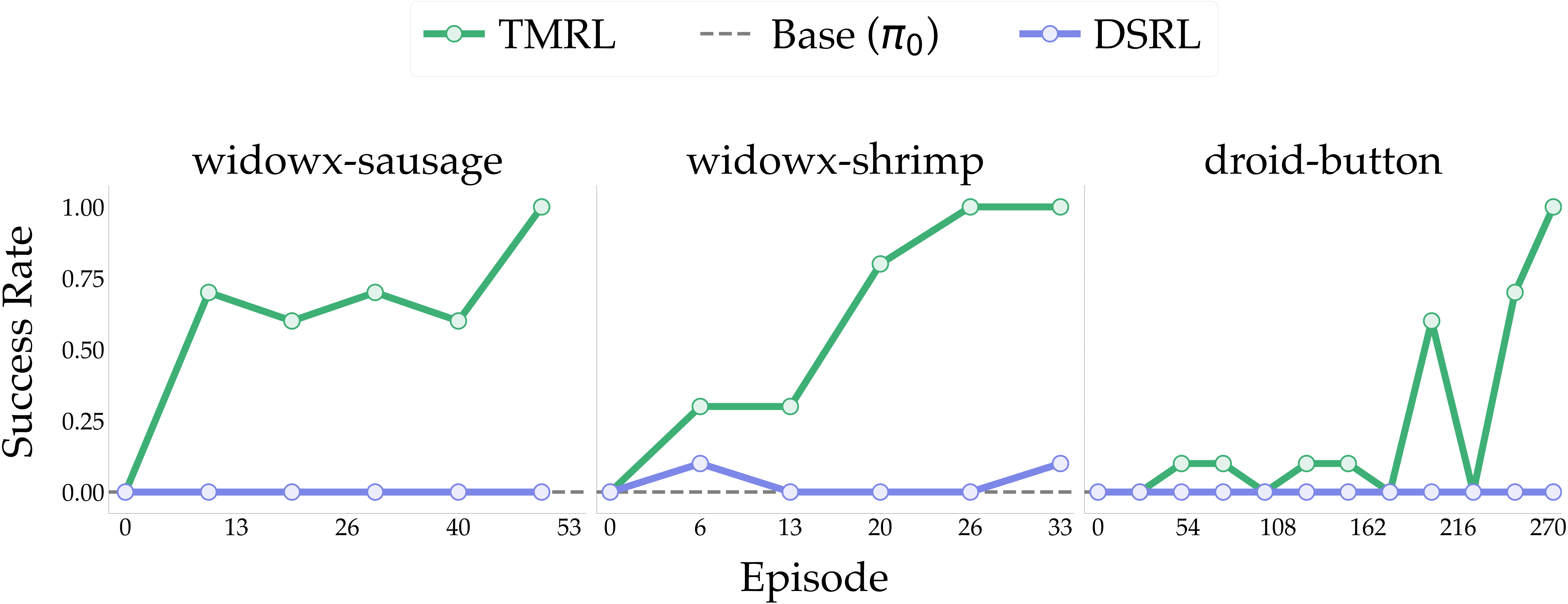}
    \caption{\tmrl\ enables steering of $\pi_0$~\cite{black2024pi0visionlanguageactionflowmodel} across three real-world tasks, while \dsrl~\citep{wagenmaker2025steering} fails to learn any task.}
    \label{fig:real}
    \vspace{-0.5cm}
\end{figure}

\subsection{\ref{q4} Analysis and Ablations.}
\label{sec:exp:ablations}
\textbf{CFG-RL Ablation:} To understand the impact of the context-smoothing as compared to other means of extrapolating beyond the base policy, we compare \tmrl\ with context-smoothed policies against those using classifier-free guidance (CFG)~\citep{ho2022classifierfreediffusionguidance}. 
CFG also interpolates between the policy marginal and conditional by tuning an interpolation coefficient $w$ on two GCPs: $\nabla_a \log p(a) + w \left(\nabla_a \log p(a \mid c) - \nabla_a \log p(a)\right)$. We directly compare this ablation (\tmrlcfg) to timestep modulation in \tmrl\, training the high-level steering policy to output the interpolation coefficient $w$ rather than the timestep. 
As seen in \Cref{fig:tmrl-cfg}, \tmrlcfg\ fails because the conditioning, which still relies on $p(a \mid c)$, struggles to extrapolate to OOD contexts no matter the $w$. In contrast, \tmrl\ brings OOD contexts back in-distribution via context corruption, providing coherent exploration in OOD settings as well.

\begin{figure}[!h]
    \centering
    \includegraphics[width=\linewidth]{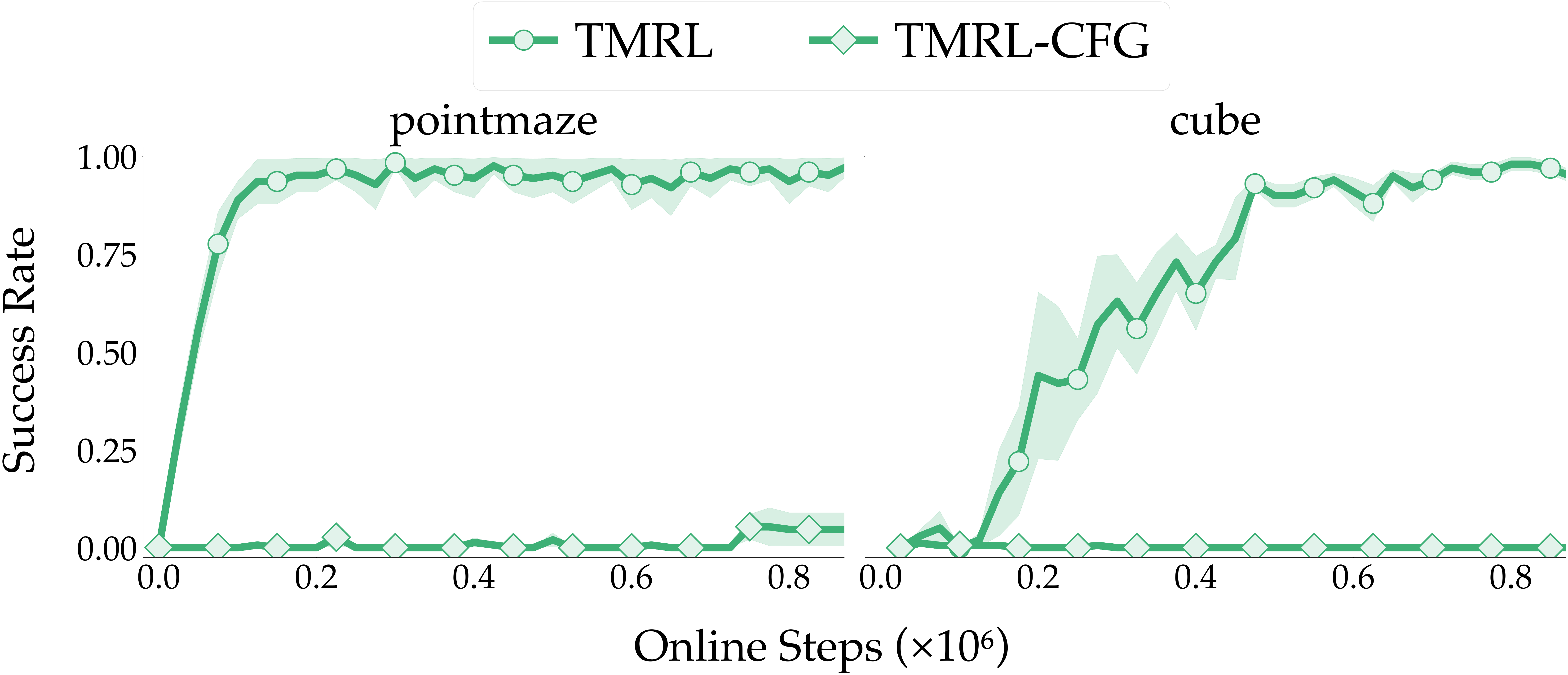}
    \caption{Comparing adaptation with \tmrl\ against \tmrlcfg\ for OGBench tasks. We find that our method substantially outperforms CFG-based interpolation.}
    \label{fig:tmrl-cfg}
    \vspace{-0.25cm}
\end{figure}

\noindent\textbf{Action Coverage Visualization:} We then visualize the exploration behavior induced by CSP+\tmrl vs. BC+\dsrl in Fig~\ref{fig:exploration}. The exploration distribution of \dsrl\ is relatively narrow, staying close to the coverage of the base policy. In contrast, the exploration behavior of \tmrl\ is considerably broader, showing a diversity of strategies beyond the narrow base policy distribution. \tmrl\ does not simply perform random actions, but rather aliases coherent actions from nearby states, for more informed exploration. 

\begin{figure}[!h]
    \centering
    \includegraphics[width=\linewidth]{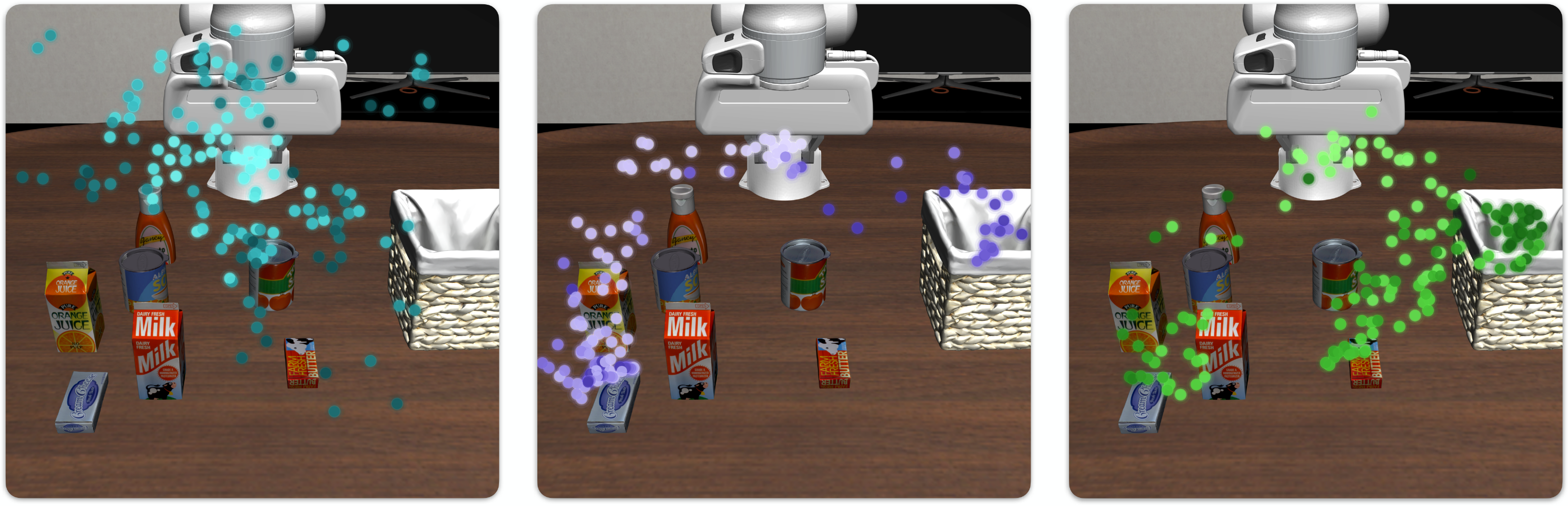}
    \caption{Comparison of exploration behaviors on the task, \texttt{"pick up the butter and put it in the basket"} for \rlpd\ (left), \dsrl\ (center), and \tmrl\ (right). While \rlpd\ explores via random action noise and \dsrl\ relies on action samples from the conditional distribution, \tmrl\ leverages action sequences from smoothed contexts, resulting in broadened, yet coherent exploration.}
    \label{fig:exploration}
\end{figure}

\noindent\textbf{Evolving Context-Smoothing Over Time:}
Finally, we demonstrate that \tmrl\ learns to dynamically adjust the amount of context-smoothing (via diffusion timestep $t_c$ in \Cref{fig:evolving_timesteps}).
At the beginning of the trajectory, \tmrl\ uses more diffusion noise so that the $\pi_0$ CSP policy can reach the sausage---by default $\pi_0$ is overfit to reaching the carrot almost every time.
Near the end of the trajectory, \tmrl\ uses less diffusion noise because once $\pi_0$ has picked up an object, it will generally attempt to correctly place it in the pot and therefore does not need as much diffusion noise to steer towards the correct behavior.
We plot more examples in Appendix \cref{fig:appendix:sausage_timesteps} and \cref{fig:appendix:shrimp_timesteps}.

\begin{figure}[htb]
    \centering
       \includegraphics[width=\linewidth]{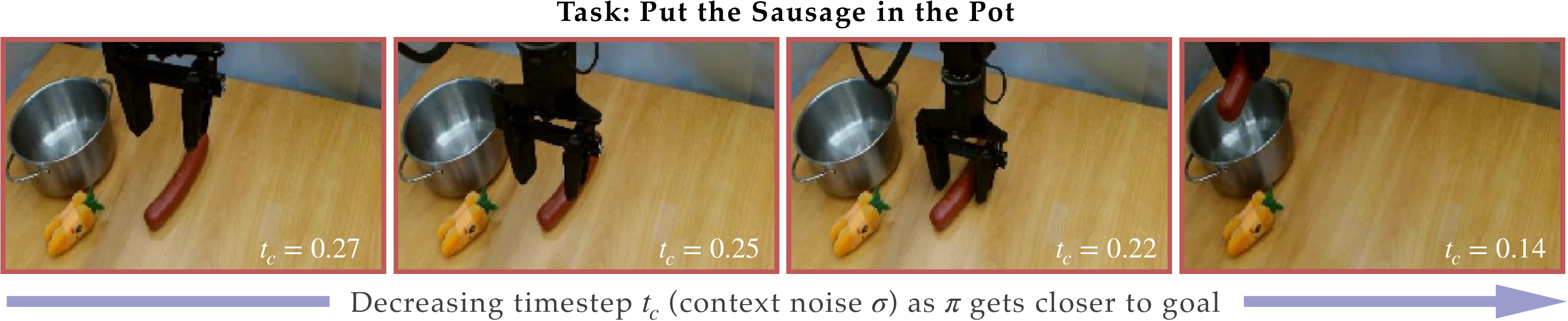}
    \caption{A rollout with a converged \tmrl\ policy on the WidowX where it learns to decrease the context-smoothing ($t_c$ normalized $[0, 1]$) throughout the trajectory.}
    \label{fig:evolving_timesteps}
\end{figure}

\section{Discussion and Limitations} 
While TMRL with context-smoothed policies can demonstrate significant performance benefits, there are many areas for improvement. Perhaps most important is safety: TMRL broadens the action distribution by aliasing actions across contexts, but this can lead to unsafe behavior. Mitigating these behaviors using a safety filter or a world model is necessary before TMRL can be used reliably. Secondly, while TMRL performs well in the real world, its sample efficiency remains impractical for many tasks of practical interest. We expect that further improvements to the steering algorithm and corruption kernel can increase TMRL's efficiency.

\section*{Acknowledgments}
We thank Chuning Zhu and Florian Shkurti for feedback on earlier drafts of this paper, and Dieter Fox for initial discussions during the project formulation.

This project was supported by funding from Amazon FAR through the Amazon Science Hub.
Additionally, we thank the UW Hyak and Tillicum high-performance computing clusters for providing us with compute resources.

\bibliographystyle{IEEEtranN}
\bibliography{refs}

\clearpage

\appendix
\section{Implementation Details}
\label{sec:appendix:implementation}

\begin{algorithm*}
\caption{Timestep-Modulated RL (TMRL) with Context-Smoothed Pretraining.}
\label{alg:tmrl_full}
\begin{algorithmic}[1]
\STATE \textbf{Pre-train context-smoothed policy:} 
Train $p_\sigma(a^{1:H} \mid c, z, \sigma)$ given context $c$, randomly sampled noise $z$ (e.g., $N(0, I)$ in diffusion policies), and smoothing parameter $\sigma$ ($=$ diffusion timestep $t_c$) on offline data $\mathcal{D}_{\text{off}}$ with \Cref{eq:train_generic}:

    \[
\min_\theta\ 
\mathbb{E}_{(c,a)\sim\mathcal{D}_\text{off}}
\ \mathbb{E}_{\substack{\sigma\sim\mathcal{S} \\ \tilde c\sim q_\sigma(\cdot\mid c)}}
\
\Big[
\ell\big(\theta;\ a, \tilde c, \sigma\big)
\Big]
\]

\STATE \textbf{Initialize:} replay buffer $\mathcal{B}$, smoothing-aware noise critic $Q(s, z, \sigma)$, high-level actor $\pi_{\mathrm{HL}}(z, \sigma \mid s)$, corruption kernel $q_\sigma$

\WHILE{interacting with the environment (or replaying from $\mathcal{B}$)}
    \STATE Update $Q$: \COMMENT{Update Smoothing-Aware Critic}
    \[
    \min_Q \mathbb{E}_{(s_t, z, r_{t:t+H}, s_{t+H}, \sigma_t) \sim \mathcal{B}, \; (\sigma', z') \sim \pi_\mathrm{HL}} \left[ Q(s, z, \sigma) - \left( \sum_{i=t}^{t+H}\gamma^ir_t +  Q(s_{t+H}, \sigma', z') \right)^2 \right]
    \]

    \STATE Update $\pi_{\mathrm{HL}}$: \COMMENT{Update High-Level (Timestep-Modulating) Actor with Off-Policy RL Objective~\citep{haarnoja2018sac}}
    \[
    \max_{z, \sigma \sim \pi_\mathrm{HL}} \mathbb{E}_{s} [ Q(s, z, \sigma) - \alpha \log \pi_\mathrm{HL}(\cdot \mid s)] 
    \]

    \IF{online environment is available}
        \STATE Sample $(z_t, \sigma_t) \sim \pi_{\mathrm{HL}}(s_t)$
        \STATE Corrupt context $\tilde{c}_t \leftarrow q_\sigma (\cdot \mid c_t)$ (\cref{eq:corruption_kernel_noise})
        \STATE Sample and execute action chunk $a^{1:H} \sim p_\theta(\cdot \mid \tilde{c}_t, z_t, \sigma_t)$
        \STATE Observe $r_{t:t+H}, s_{t+H}$, and add $(s_t, z_t, r_{t:t+H}, s_{t+H}, \sigma_t)$ to $\mathcal{B}$
    \ENDIF
\ENDWHILE

\end{algorithmic}
\end{algorithm*}

\subsection{Extended Proof of Thm 1}
\label{sec:appendix:proof}

\textbf{Theorem 1.}
We restate \Cref{thm:tv_lipschitz_overlap} for clarity:
For any two contexts $c,c'\in\mathbb{R}^d$ (\Cref{eq:tv_lipschitz}),
\begin{equation*}
\mathrm{TV}\!\big(p_\sigma(\cdot\mid c),p_\sigma(\cdot\mid c')\big)
~\le~
\frac{\mathbb{E}\|w\|}{\sigma}\,\|c-c'\|.
\end{equation*}
Equivalently (\Cref{eq:ov_lipschitz}),
\begin{equation*}
\mathrm{Ov}\!\big(p_\sigma(\cdot\mid c),p_\sigma(\cdot\mid c')\big)
~\ge~
1-\frac{\mathbb{E}\|w\|}{\sigma}\,\|c-c'\|.
\end{equation*}
In particular, for fixed $c,c'$ the lower bound in \Cref{eq:ov_lipschitz} is nondecreasing in $\sigma$,
and for fixed $\sigma$ it is nonincreasing in $\|c-c'\|$.

Our proof builds on the Gaussian-smoothing analysis originally developed for randomized smoothing~\citep{cohen2019certified}. \citet{block2023provable} prove an analogous smoothness guarantee for observation-noised policies, with the same $1/\sigma$ scaling, via a different argument.

\begin{proof}
For any measurable action event $A \subseteq \mathcal{A}$, define
\begin{equation}
    f_A(c) := p(A \mid c) \in [0, 1].
\end{equation}

By definition of a context-smoothed conditional action distribution $p_\sigma$ in \Cref{eq:gaussian_smoothing_def},

\begin{equation}
    p_\sigma(A \mid c) = \mathbb{E} [f_A(c+\sigma w)], \; \; w \sim \mathcal{N}(0, I_d).
\end{equation}

Let $g_A(c) := p_\sigma(A \mid c)$. Then using Stein's identity (Gaussian smoothing gradient identity),
\begin{equation}
    \nabla g_A(c)  = \frac{1}{\sigma} \mathbb{E}[f_A(c + \sigma w)w].
\end{equation}

Thus, 
\begin{equation}
\|\nabla g_A(c) \| \le \frac{1}{\sigma} \mathbb{E}[|f_A(c + \sigma w)| \, \|w\|] \le \frac{\mathbb{E}\|w\|}{\sigma}.
\label{eq:appdx:lipschitz}
\end{equation}
The first inequality follows by Jensen's inequality and the convexity of $|| \cdot ||$, and the second from $0 \le f_A \le 1.$ 
Hence \Cref{eq:appdx:lipschitz} implies that $g_A$ is ($\mathbb{E}\|w\|/\sigma$)-Lipschitz continuous. 
Thus, for two contexts $c, c'$ and their difference $c - c'$:

\begin{equation}
    |p_\sigma(A \mid c) - p_\sigma (A \mid c')| = |g_A(c) - g_A(c')| \le \frac{\mathbb{E} \|w\|}{\sigma} \| c - c'\|.
\end{equation}

Taking $\sup_A$ gives \Cref{eq:tv_lipschitz}. Then \Cref{eq:ov_lipschitz} follows simply from the overlap identity defined in \Cref{sec:method}: $\mathrm{Ov}(P,Q)\coloneqq 1-\mathrm{TV}(P,Q)$.
\end{proof}

\subsection{Extended Proof of \Cref{cor:overlap_threshold}}
\label{sec:appendix:proof_corollary}
At a high level, \Cref{cor:overlap_threshold} states that for two contexts $c, c'$, smoothing increases overlap even with respect to the conditional distribution of the original policy $p(\cdot \mid c), p(\cdot \mid c')$, extending the result from \Cref{thm:tv_lipschitz_overlap} that smoothing increases overlap with respect to the conditional distribution of the context-smoothed policy $p_\sigma$.

\begin{proof}
    Assume the identifiability lower bound from \Cref{eq:identifiability}: $\mathrm{TV}\big(p(\cdot \mid c), p(\cdot \mid c')\big) \ge m\|\Delta\|$ for some $m> 0$, $\Delta = c - c'$
    Assuming $\sigma \ge \mathbb{E}\|w\|/m$ as stated in the assumption \Cref{eq:sigma_threshold}. Then, following \Cref{eq:tv_lipschitz} from \Cref{thm:tv_lipschitz_overlap}, we have:
    \begin{align*}
        \mathrm{TV}\big( p_\sigma(\cdot \mid c), p_\sigma(\cdot \mid c')\big) &\le \frac{\mathbb{E}\|w\|}{\sigma} \|\Delta\| \\
        & \le m \|\Delta\| \\
        &\le \mathrm{TV}\big( p(\cdot \mid c), p(\cdot \mid c')\big).
    \end{align*}

    Converting $\mathrm{Ov} = 1 - \mathrm{TV}$ yields \Cref{eq:overlap_increase}.
\end{proof}

\subsection{\method\ Implementation details}
\label{sec:appendix:impl_details}
\textbf{Training a context-smoothed imitation learning policy.} Our approach is agnostic to the choice of policy class; however, in this work we focus on generative flow and diffusion policies, motivated by their recent success in behavior cloning for robotic learning.
For training our own context-smoothed policies in OGBench (not widowx or dexterous tasks), our noise prediction network uses a Diffusion Transformer (DiT) backbone adapted from \citet{zhu2025uwm}. While typical GCPs model $p(a\mid c)$, training a context-smoothed policy differs only in sampling a context timestep and corrupting the context $\tilde{c}$ (\Cref{eq:corruption_kernel_noise}). The actions $a_t$ are embedded and fed as a sequence into the transformer. The noisy conditionings, action timestep, the conditioning timestep and other observation conditionings that are not being noised (\eg robot proprioception) are used to condition the transformer via Adaptive Layer Normalization (AdaLN)~\citep{peebles2023scalable}. We train the model with \Cref{eq:train_generic}

\begin{table}[t]
    \centering
    \caption{Flow policy training hyperparameters}
        \begin{tabular}{ll}
        \hline
        \textbf{Hyperparameter} & \textbf{Value} \\
        \hline\\[-5pt]
        \textbf{Model} \\[2pt]
        Hidden size             & 128 \\
        Embedding dim           & 256 \\
        Activation              & Mish \\
        Number of layers        & 4 \\
        Number of heads         & 4 \\[2pt]
        \textbf{Flow} \\[2pt]
        Training flow steps     & 100 \\
        Inference flow steps    & 10 \\
        Timestep sampling       & Uniform \\[2pt]
        \textbf{Training} \\[2pt]
        Batch Size              & 256 \\
        Optimizer               & AdamW \\
        Learning rate           & $3e^{-4}$ \\
        Weight decay            & $1e^{-8}$ \\
        Betas                   & $[0.95, 0.999]$ \\
        Epsilon                 & $1e^{-8}$ \\
        LR schedule             & constant \\
        \hline
        \end{tabular}
    
    \label{tab:hp_train}
\end{table}

Our corruption kernel $q$, uses a discrete linear scheduler with \texttt{beta\_start}$=1e^{-4}$, \texttt{beta\_end}$=0.02$ with $T=1000$ diffusion timesteps. During training, we sample a continuous $t\sim Unif(0, 1)$ which we note is sampled independently from our action noise timestep $t_a$ (for clarity, we denote $t_c=t$ for the context timestep). We then sample noise, $\epsilon\sim \mathcal{N}(0,I)$, convert the timestep to a discrete index, and then sample the corrupted context $c_t=\sqrt{\bar{\alpha}_t}c+\sqrt{1-\bar{\alpha_t}}\epsilon$.
The policy is trained to model the action distribution $p(a\mid c_t, t)$ at different levels of noise $t$ conditioned on the noisy context $c_t$. In our experiments, we show that context-smoothed policy training can be done from scratch or through fine-tuning on a base VLA model.

\textbf{Timestep-Modulated Reinforcement Learning.}
Our experiments use SAC~\citep{haarnoja2018sac} or RLPD~\citep{pmlr-v202-ball23a} as the online RL algorithm. Building on top of DSRL, we run RL over a high-level actor whose outputs are fed into the low-level GCP. Our high level actor $\pi_{\text{HL}}$ outputs both the initial noise vector $z$ as well as the noise timestep $t_c$. We model the actor as two independent distributions, $\pi_{\text{HL}}^z$ and $\pi_\text{HL}^{t_c}$, which allows us to use different noise bounds as well as target entropies for both distributions. Following DSRL, we also train separate timestep-aware noise critics with $t_c$ as an additional input $Q(s,a,z,t_c)$. In our experiments, we experiment with and without the timestep-aware noise critic and find no major performance differences. Our \texttt{OGBench} experiments utilize both critics while the rest of our experiments, use noise critics $Q(s,a,z,t_c)$.
\begin{table}[ht]
    \centering
    \caption{Common \method\ hyperparameters}
        \begin{tabular}{ll}
        \hline
        \textbf{Hyperparameter} & \textbf{Value} \\
        \hline\\[-5pt]
        Critic LR               & $3e^{-4}$ \\
        Actor LR                & $3e^{-4}$ \\
        Autotune                & True \\
        Tau                     & 0.005 \\
        Critics                 & 5 \\
        Actor and critic layers & 3 \\
        Buffer warmup steps     & 500 \\
        \hline
        \end{tabular}
    
    \label{tab:hp_rl}
\end{table}

\subsection{Baseline Implementation Details}
\label{sec:appendix:baseline_details}
\textbf{RLPD.} Following \citet{pmlr-v202-ball23a}, which we also apply to all other methods: each batch consists of 50\% offline data and 50\% samples from the online replay buffer; critics use layer normalization, we subsample two critics, and remove entropy backups.


\textbf{SPiRL.} While SPiRL~\citep{pertsch2020spirl} initially proposes to utilize an LSTM~\citep{10.1162/neco.1997.9.8.1735} as the architecture choice for the skill encoder and decoder, we adopt a more modern transformer architecture. After pre-training the skill prior, we initialize the learnable actor with the skill prior.
Following \citet{pertsch2020spirl}, we minimize the reverse KL $D_{KL}(q || p)$ between the skill prior and the actor. Using the reverse KL rather than the forward KL ensures that the learned prior is mode-covering.

\textbf{PostBC.} Following \citet{wagenmaker2025posteriorbehavioralcloningpretraining}, we first train an ensemble of MLP predictors on bootstrapped resamples of the demonstration dataset to approximate the posterior covariance of the demonstrator's actions. Each ensemble member is trained via regression on a bootstrap resample (sampled with replacement) of the dataset. The per-state posterior covariance is estimated as the empirical variance across ensemble predictions. During BC pretraining, we perturb each action target in each training batch by noise sampled from a zero-mean Gaussian with this estimated posterior covariance, scaled by a weight $\alpha$. The final policy is parameterized as a diffusion model trained on these noise-perturbed action targets. At inference time, the policy is used without modification.

\begin{figure*}[t]
    \centering
    \includegraphics[width=0.9\textwidth]{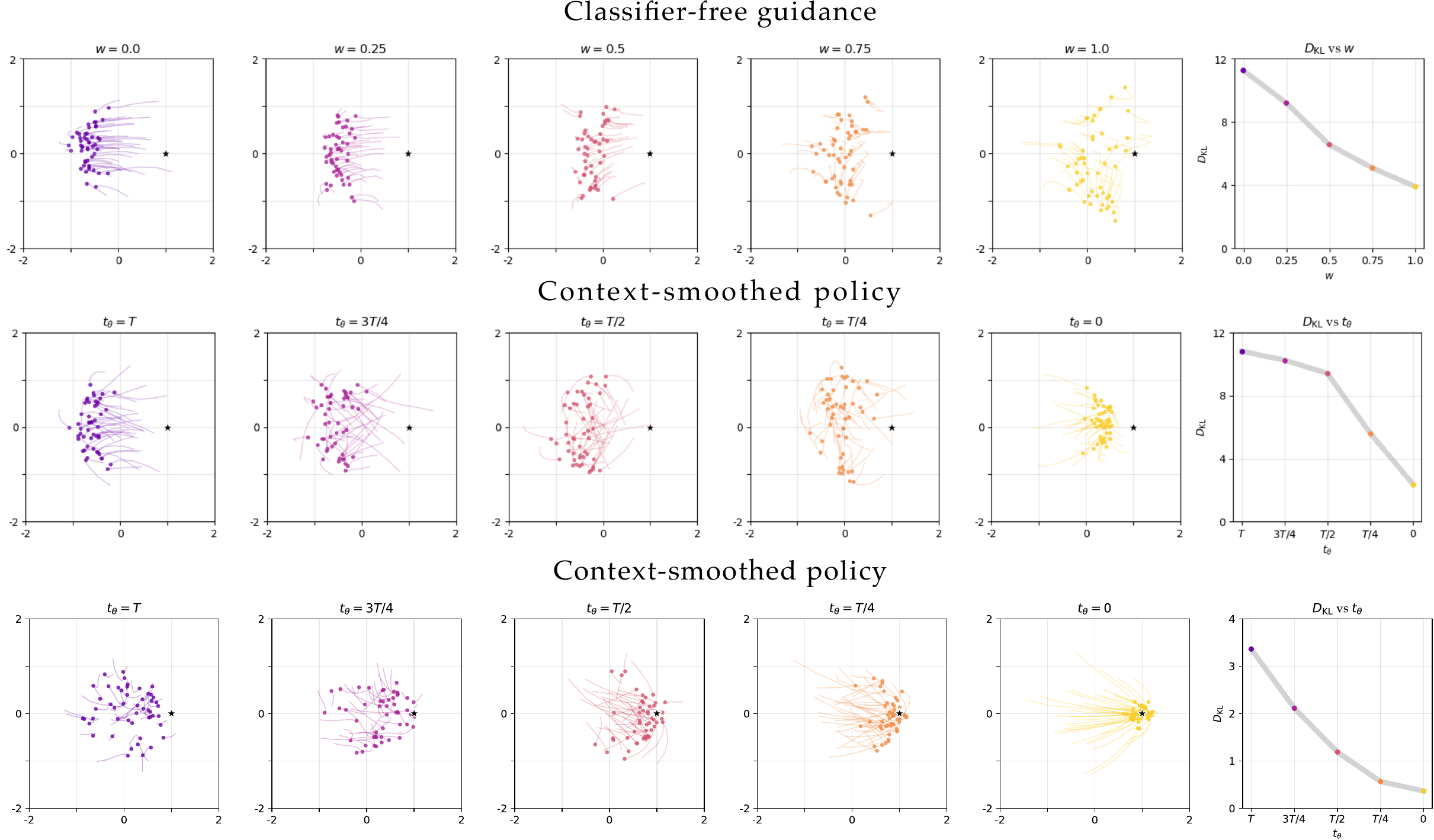}
    \caption{\textbf{Conditioned smoothed policies enable out-of-support generalization.} (Top) In a 2D unit-circle setup (\Cref{fig:borrowing-exp}), we restrict training to angles $\theta \in [\pi/2, 3\pi/2]$ and evaluate on an unseen angle $\theta = 0$. Unlike standard conditional diffusion, context-smoothed training (\Cref{eq:train_generic}) successfully approximates the target Gaussian at $(\cos 0, \sin 0)$. 
    \textbf{Noisy Conditioning Enables Marginal-to-Conditional Interpolation.} (Bottom) Angles on the unit circle define the conditioning variable $\theta$, with corresponding 2D points $(x, y) = (\cos\theta, \sin\theta)$. The marginal distribution $p(x,y)$ consists of noisy points around the circle, while the conditional distribution $p(x,y \mid \theta)$ is a Gaussian at $N([\cos\theta, \sin\theta], I)$. By varying the conditioning noise level $t_\theta$ for $\theta = 0$, we visualize denoised samples: larger $t_\theta$ yields weaker conditioning (approaching the marginal $p(x,y)$), and smaller $t_\theta$ yields stronger conditioning (approaching $p(x,y \mid \theta)$). The rightmost plots report the KL divergence between the induced distribution at each $t_\theta$ and the target conditional $p(x,y \mid \theta=0)$, highlighting our method's improved OOD generalization.}
    \label{fig:toy-exp}
\end{figure*}

\subsection{TMRL vs. CFG-RL}
At their core, \method\ and CFG~\citep{ho2022classifierfreediffusionguidance} represent independent ways to sample from distributions that control the reliance on a context $c$. CFG jointly trains a conditional $p(a\mid c)$ and unconditional model $p(a)$ (typically sharing parameters), and at inference time interpolates between them using classifier-free guidance. Concretely, the guided score is formed as 
$\nabla_a \log{\pi(a | s)} + w(\nabla_a \log{\pi(a | s, o)} - \nabla_a\log{\pi(a | s)})$. In \Cref{fig:toy-exp}, we visualize a toy example comparing sampled points of an unseen conditioning $\theta$. We observe that CFG approximates the true distribution worse than a context-smoothed policy (here, our noisy-conditioning approach). We attribute this to CFG not being trained on this out of distribution conditioning. CFG samples have higher variance and are more spread out at $w=1.0$. Whilst, during context-smoothed policy training, the novel $\theta$ has likely been seen during training, thus the policy generates samples that more closely approximate the true distribution. 

\subsection{Details of Simulation Experiments}
\label{sec:appendix:sim_exps}
\textbf{OGBench.} For our first set of simulation experiments in OGBench~\citep{ogbench_park2025}, we utilize a flow-matching policy~\citep{lipman2022flow}. We use RLPD~\citep{pmlr-v202-ball23a} as our RL algorithm, which can leverage prior data alongside online interaction data. The same offline data used to pre-train the diffusion policy is used during online learning where we employ 50/50 sampling (i.e., 50\% of the data comes from the offline data and 50\% of the data comes from the online replay buffer). Additionally, for all methods and experiments, we utilize 5 critics, layer norms, and remove the entropy term from the critic loss. For both evaluation environments, the agent is required to learn multiple tasks simultaneously: four for PointMaze and three for Cube. Both PointMaze and Cube employ sparse rewards: the agent receives a reward of 0 upon successful episode termination and –1 otherwise. 

\begin{figure}[h]
    \centering
    \includegraphics[width=\linewidth]{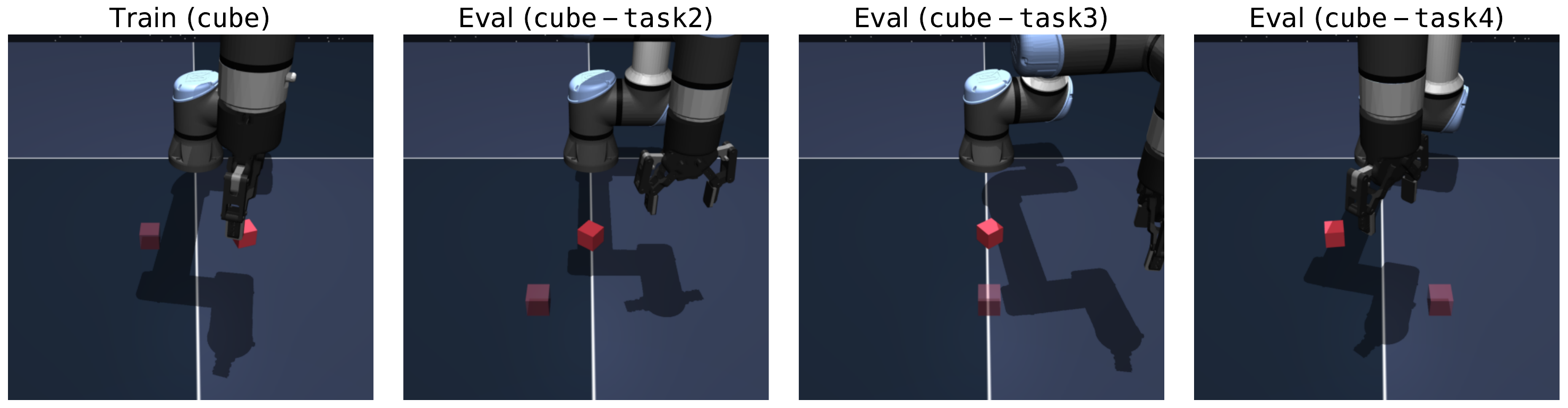}
    \caption{\textbf{Cube environments.} We filter transitions from the \texttt{cube-single-play} dataset where the cube or the cube goal x-position is $> 0.4$. We evaluate on three tasks where the goal location is located beyond the filtered region.}
    \label{fig:cube_env}
\end{figure}

\begin{figure*}[t]
    \centering
    \includegraphics[width=\linewidth]{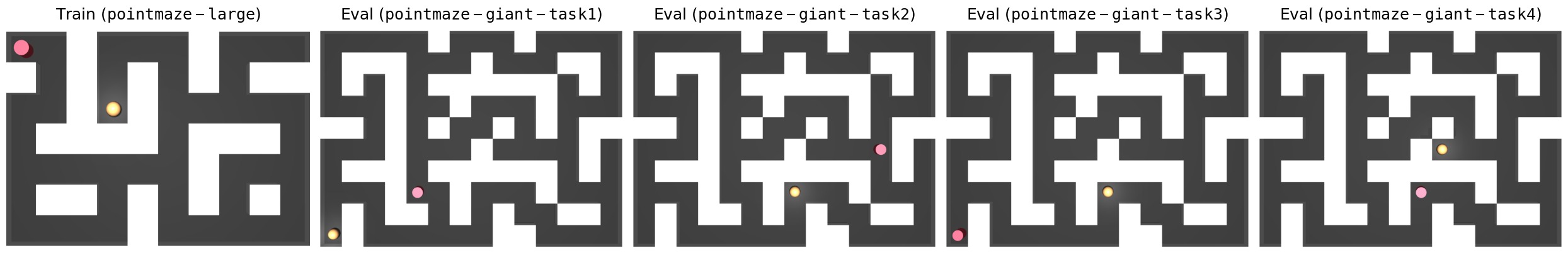}
    \caption{\textbf{PointMaze environments.} Example train and evaluation mazes. The evaluation environment features a significantly larger maze layout. We train RL policies on four out-of-distribution goal locations simultaneously.}
    \label{fig:pointmaze_env}
\end{figure*}

\begin{table}[ht]
    \centering
    \caption{OGBench hyperparameters}
        \begin{tabular}{ll}
        \hline
        \textbf{Hyperparameter} & \textbf{Value} \\
        \hline\\[-5pt]
        Action len              & $50$ (\texttt{pointmaze}), $50$ (\texttt{cube}) \\
        Discount                & $0.995$ \\
        Noise bound             & $1.0$ \\
        Hidden dim              & $512$ \\
        Action target entropy   & $-\texttt{action\_dim}$ \\
        Timestep target entropy & $-1$ \\
        Timestep bound          & $1.0$ \\
        Num Envs                & $4$ \\
        \hline
        \end{tabular}

    \label{tab:hp_ogbench}
\end{table}

We use the \texttt{pointmaze-large-navigate-v0} dataset for training, where each observation consists of the agent’s 2D position concatenated with the 2D goal. This combined vector serves as the conditioning input that is subjected to noise. Evaluation is performed on the larger \texttt{pointmaze-giant-navigate} environment, as shown in \Cref{fig:pointmaze_env}. We use the \texttt{cube-single-play} dataset, filtering out transitions where the cube’s position exceeds a predefined threshold, as visualized in \Cref{fig:cube_env}. The perturbed conditioning variable is the cube’s three-dimensional goal position. Evaluation is conducted on three tasks with goal locations situated in the out-of-distribution region.


\textbf{Libero.} For our Libero~\citep{liu2023libero} experiments, we utilize the $\pi_0$-Libero checkpoint provided by \citet{black2024pi0visionlanguageactionflowmodel} at \texttt{s3://openpi-assets/checkpoints/pi0\_libero} which trains on four Libero datasets: \texttt{Libero-Spatial}, \texttt{Libero-Object}, \texttt{Libero-Goal} and \texttt{Libero-10}. For TMRL, we also fine-tune $\pi_0$-base using context-smoothed policy training on the same datasets. The other baselines also use the same dataset. We evaluate on two novel tasks: task 51 from the \texttt{Libero-90} task suite and the swap perturbation for $\texttt{Libero-Goal}$ from \citet{zhou2025liberopro}. 

\begin{figure}[h]
    \centering
    \includegraphics[width=\linewidth]{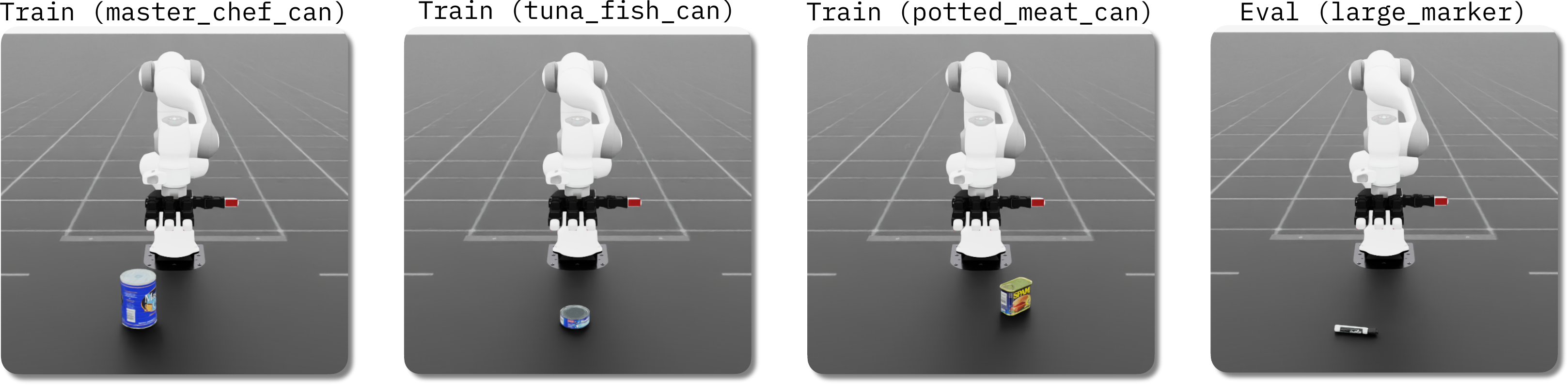}
    \caption{\textbf{Dexterous manipulation grasping objects.} A policy is trained on three objects \texttt{master\_chef\_can} , \texttt{tuna\_fish\_can} and \texttt{potted\_meat\_can} and then evaluated on a novel \texttt{large\_marker} object.}
    \label{fig:appendix:dex_objects}
\end{figure}

\textbf{Dexterous grasping.} Our dexterous grasping experiments are run in simulation built on IsaacLab~\citep{mittal2025isaaclab}, we distill an RL expert into point cloud conditioned visuomotor policies via a standard student-teacher approach \citep{pmlr-v164-chen22a} and train on three can-shaped objects---\texttt{master\_chef\_can}, \texttt{tuna\_fish\_can}, \texttt{potted\_meat\_can}---collecting 200 demonstrations for each and evaluate adaptation to the unseen \texttt{large\_marker} object shown in \Cref{fig:appendix:dex_objects}. We first train a diffusion policy~\citep{chi2023diffusionpolicy} using point clouds as the noised context, we also condition on the 22-dimensional joint pose consisting of 6-dimensional arm pose and the 16-dimensional finger joint pose.  We use a relative joint pose action space. After training a context-smoothed policy, we use PPO~\citep{schulman2017proximalpolicyoptimizationalgorithms} as our algorithm and run TMRL over the frozen base policy.

\begin{table}[h]
    \centering
    \caption{Dexterous Grasping PPO hyperparameters}
        \begin{tabular}{ll}
        \hline
        \textbf{Hyperparameter} & \textbf{Value} \\
        \hline\\[-5pt]
        Number of envs & $128$ \\
        Batch Size              & $4096$ \\
        Discount               & $0.99$ \\
        GAE lambda & $0.95$ \\
        LR & $3e^{-4}$ \\
        Clip range & $0.2$ \\
        Value loss coefficient & $1e^{-4}$ \\
        Max gradient norm & $1.0$ \\
        Optimizer & Adam \\
        \hline
        \end{tabular}
    \label{tab:hp_ppo}
\end{table}

\subsection{Details of Real Experiments}
\label{sec:appendix:real_exps}
\textbf{WidowX.} For our real-world WidowX experiments, we utilize the $\pi_0$ VLA model provided by \citet{black2024pi0visionlanguageactionflowmodel} at \texttt{s3://openpi-assets/checkpoints/pi0\_base}. We fine-tune the model on the Bridge V2~\citep{pmlr-v229-walke23a} dataset which is used for DSRL and fine-tune a separate context-smoothed policy for TMRL. We perform full fine-tuning of $\pi_0$-Base with 2 Nvidia H200 GPUs, training for 20 hours. We list further training details in \Cref{tab:hp_pi0}.

We evaluate our fine-tuned policies on two unseen tasks: \texttt{widowx-sausage} and \texttt{widowx-shrimp}. Both use maximum episode lengths of 100 steps and use sparse rewards (-1/0) rewarding the agent only when the task is completed. The \texttt{widowx-sausage} task consists of a sausage, a pot, and a distractor carrot object. The objective is to pick and place the sausage inside of the pot. We randomize all object positions within a 10 cm radius. The \texttt{widowx-shrimp} task consists of three drawers (white, red, green), a distractor sushi object on a green plate, and a shrimp object on a white plate. The objective is to pick and place the shrimp inside of the white drawer. We randomize the sushi and shrimp object positions within a 10 cm radius as well as randomizing orientation of both objects. 

\textbf{DROID.} For our real-world DROID experiments, we utilize the $\pi_0$ VLA model provided by \citet{black2024pi0visionlanguageactionflowmodel} at \texttt{s3://openpi-assets/checkpoints/pi0\_base}. We fine-tune a context-smoothed policy on the DROID~\citep{khazatsky2024droid} dataset and utilize the $\pi_0-\text{DROID}$ VLA model provided by \citet{khazatsky2024droid} at \texttt{s3://openpi-assets/checkpoints/pi0\_droid} which is used for DSRL and fine-tune a separate context-smoothed policy for TMRL. We perform full fine-tuning of $\pi_0$-Base with 4 Nvidia H200 GPUs for 100k steps, which takes roughly 3 days. We list further training details in \Cref{tab:hp_pi0}.

\begin{figure*}[h]
    \centering
    \includegraphics[width=\linewidth]{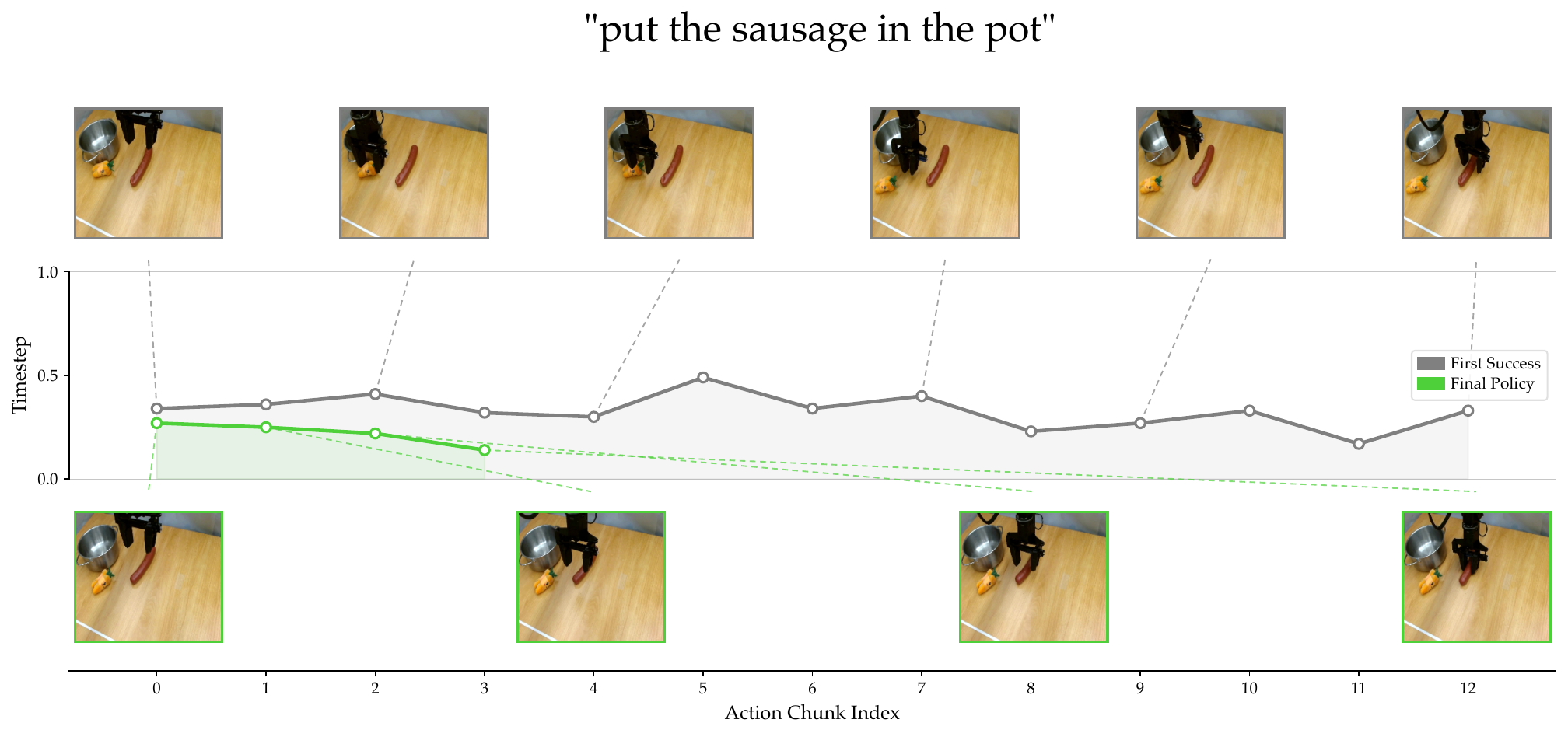}
    \caption{\textbf{Timestep modulation over the course of a rollout.} We visualize the timesteps used by TMRL across action chunk indices for the task ``put the sausage in the pot,'' comparing the first successful rollout to the converged final policy. The final policy uses lower timesteps toward the end of the trajectory, reflecting reliance on increasingly precise, imitation-like action distributions as the task nears completion.}
    \label{fig:appendix:sausage_timesteps}
\end{figure*}

\begin{figure*}[h]
    \centering
    \includegraphics[width=\linewidth]{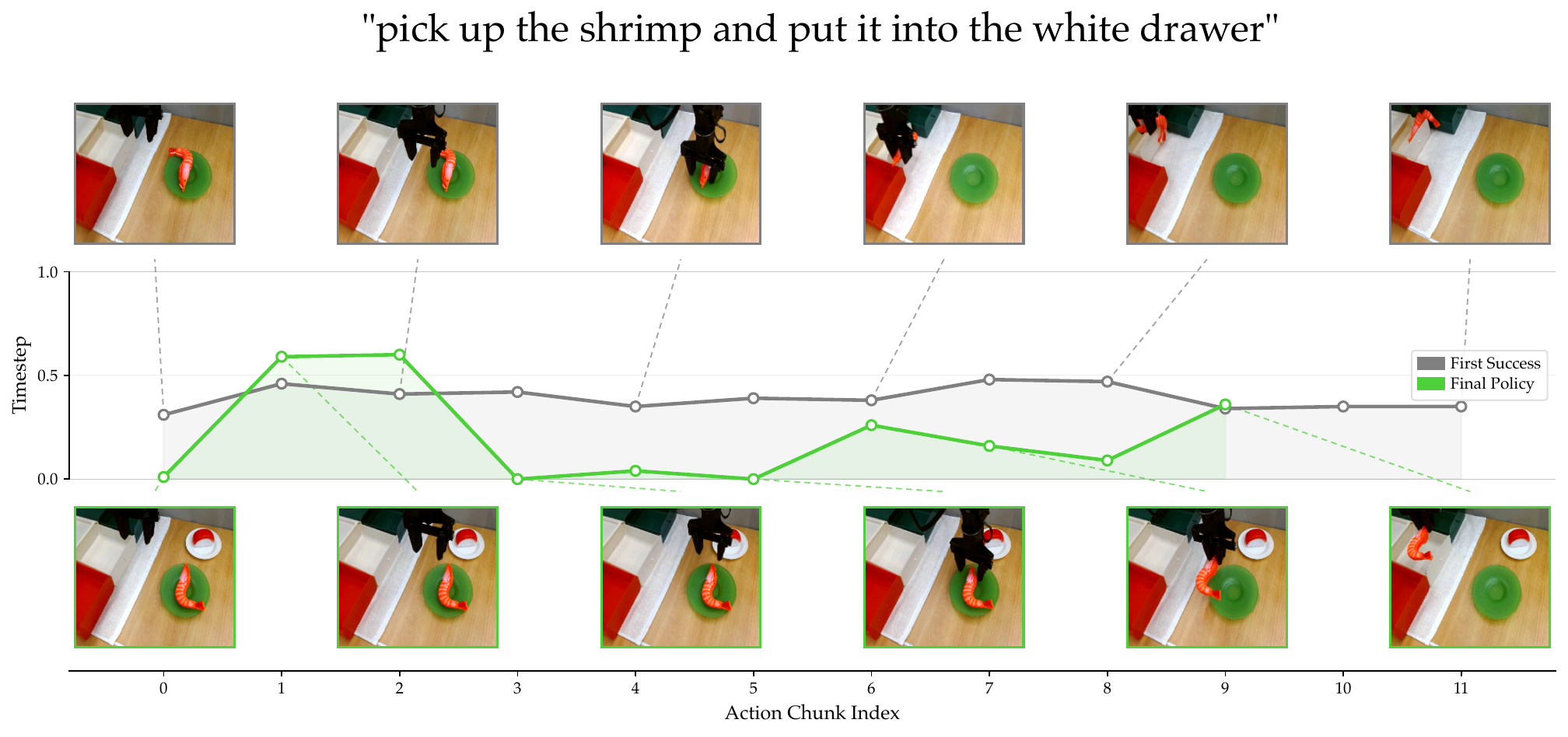}
    \caption{\textbf{Timestep modulation over the course of a rollout.} We visualize the timesteps used by TMRL across action chunk indices for the task ``pick up the shrimp and put it into the white drawer,'' comparing the first successful rollout to the converged final policy. The final policy uses notably lower timesteps during the manipulation-critical middle portion of the trajectory, reflecting more precise, imitation-like action distributions during the pick-and-place motion.}
    \label{fig:appendix:shrimp_timesteps}
\end{figure*}

We follow standard practices \citep{wagenmaker2025steering, pmlr-v202-ball23a}, such as high UTDs, layer normalizations in the critic networks, deep critic networks, and the use of multiple critics, we list additional hyperparameters in \Cref{tab:pi_ft}. For our real-world tasks, we generally use lower action noise bounds for TMRL than DSRL as we observe that using higher action bounds for TMRL tends to generate out-of-distribution and unsafe action sequences. 
A possible explanation for this is that noisy contexts are more likely to be out-of-distribution for the action expert. In addition, naively sampling action distributions close to the marginal $p(a)$ results in unsafe action sequences. To mitigate this issue, we simply upper bound the maximum timestep that can be sampled by the high-level policy $\pi_{\text{HL}}$ and observe that this mitigates such undesirable behaviors. 

\begin{table}[h]
    \centering
    \caption{Finetuning $\pi_0$ Libero and Bridge  hyperparameters}
        \begin{tabular}{ll}
        \hline
        \textbf{Hyperparameter} & \textbf{Value} \\
        \hline\\[-5pt]
        \textbf{Training} \\[2pt]
        Batch Size              & $32$ \\
        Optimizer               & cosine \\
        EMA Decay               & $0.99$ \\
        Decay LR                & $2.5e^{-6}$ \\
        Peak LR                 & $5e^{-4}$ \\
        Warmup steps            & $1e^3$ \\
        Train steps             & $3e^4$ \\
        \textbf{Flow} \\[2pt]
        Inference flow steps    & $10$ \\
        \hline
        \end{tabular}
    \label{tab:pi_ft}
\end{table}

DSRL~\citep{wagenmaker2025steering} originally uses VLM embeddings from the PaliGemma VLM model as inputs to both the actor and critic. However in our experiments, we find that replacing these with image embeddings encoded by DINOv2~\citep{oquab2024dinov2learningrobustvisual} performs comparably with far less computational overhead. We hypothesize that the advantage of VLM embeddings becomes more pronounced in multi-task settings with diverse language prompts, where language-conditioned representations can better capture task-specific structure.

\begin{table}[h]
    \centering
    \caption{Hyperparameters for TMRL $\pi_0$ experiments}
    \begin{tabular}{lccc}
        \hline
        \textbf{Hyperparameter} & \textbf{Libero} & \textbf{WidowX} & \textbf{DROID} \\
        \hline\\[-5pt]
        Discount                & $0.99$  & $0.99$ & $0.99$ \\
        Action execution len    & $5$     & $5$    & $5$    \\
        Noise bound             & $1.0$   & $0.25$ & $1.0$  \\
        Prefix noise bound      & $1.0$   & $1.0$  & $1.0$  \\
        Hidden dim              & $1024$  & $1024$ & $512$  \\
        Action target entropy   & $-16$   & $-16$  & $-16$  \\
        Timestep target entropy & $-0.5$  & $-0.5$ & $-0.5$ \\
        Timestep bound          & $0.6$   & $0.6$  & $0.1$  \\
        UTD                     & $1$     & $20$   & $20$   \\
        Num Envs                & $4$     & $1$    & $1$    \\
        \hline
    \end{tabular}
    \label{tab:hp_pi0}
\end{table}

\begin{table}[h]
    \centering
    \caption{Hyperparameters for DSRL $\pi_0$ experiments}
    \begin{tabular}{lccc}
        \hline
        \textbf{Hyperparameter} & \textbf{Libero} & \textbf{WidowX} & \textbf{DROID} \\
        \hline\\[-5pt]
        Discount                & $0.99$  & $0.99$ & $0.99$ \\
        Action execution len    & $5$     & $5$    & $5$    \\
        Noise bound             & $1.0$   & $1.0$ & $1.0$  \\
        Hidden dim              & $1024$  & $1024$ & $512$  \\
        Action target entropy   & $-16$   & $-16$  & $-16$  \\
        UTD                     & $1$     & $20$   & $20$   \\
        Num Envs                & $4$     & $1$    & $1$    \\
        \hline
    \end{tabular}
    \label{tab:hp_pi0_dsrl}
\end{table}

\begin{figure*}[ht]
    \centering
    \includegraphics[width=0.65\linewidth]{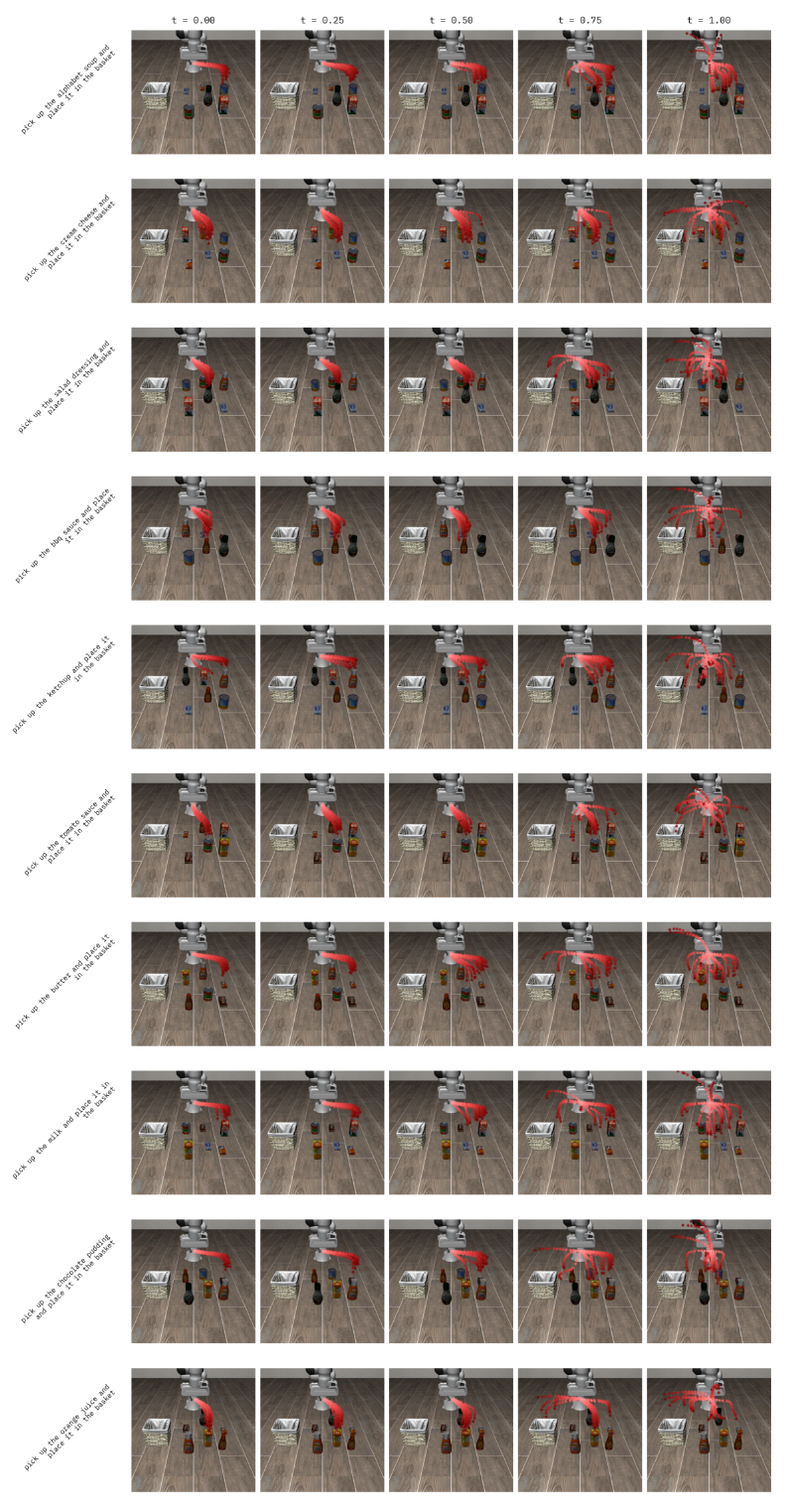}
    \caption{\textbf{\texttt{Libero-Object} action distributions under context smoothing.} Each row corresponds to a distinct task in the \texttt{Libero-Object} suite. Columns show rollouts as the observation embedding is progressively smoothed from the original context (left) toward the marginal distribution (right), interpolating between the conditional and marginal policy.}
    \label{fig:libero_object_smoothing}
\end{figure*}

\begin{figure*}[ht]
    \centering
    \includegraphics[width=0.65\linewidth]{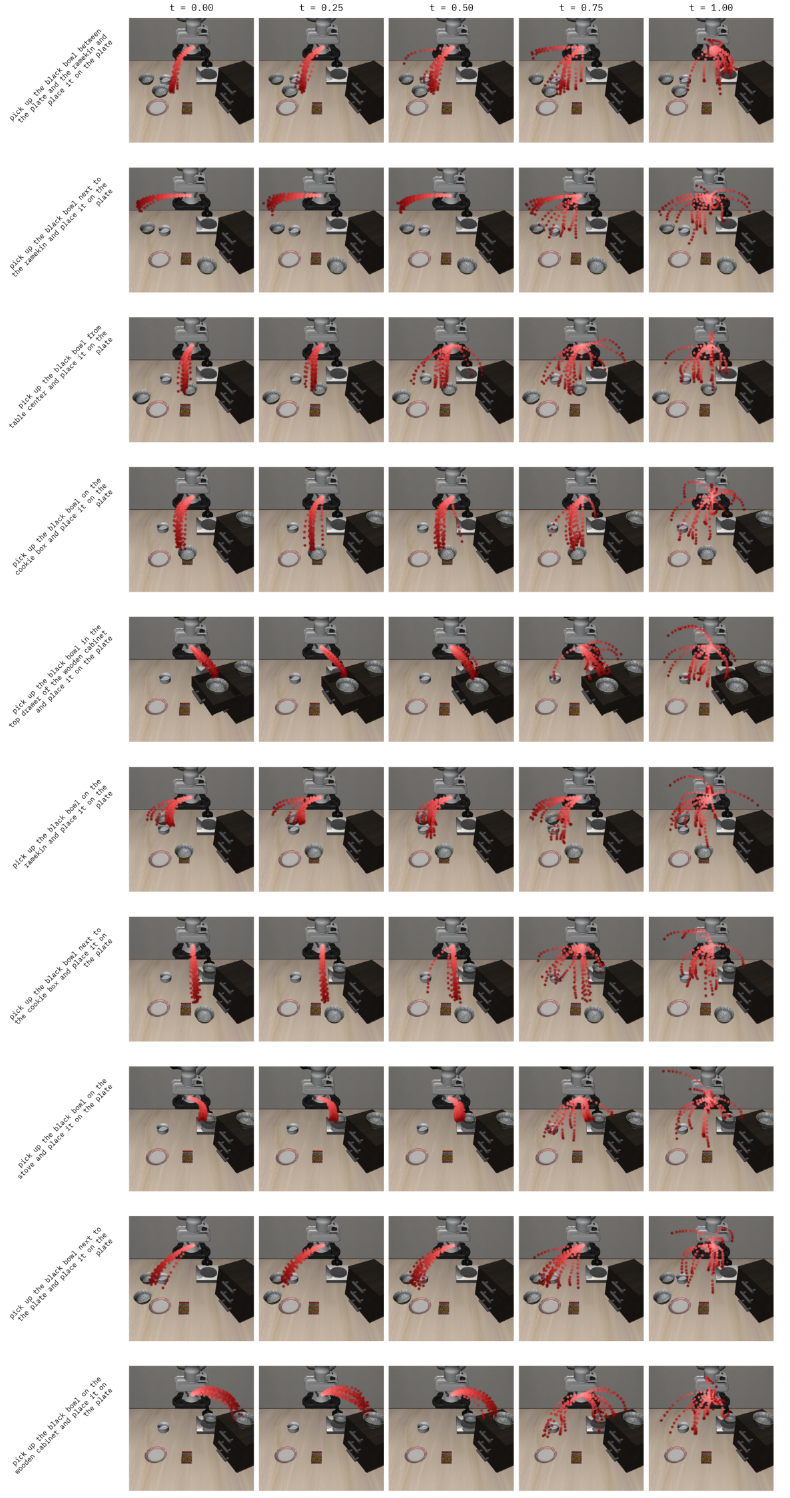}
    \caption{\textbf{\texttt{Libero-Spatial} action distributions under context smoothing.} Each row corresponds to a distinct task in the \texttt{Libero-Spatial} suite. Columns show rollouts as the observation embedding is progressively smoothed from the original context (left) toward the marginal distribution (right), interpolating between the conditional and marginal policy.}
    \label{fig:libero_spatial_smoothing}
\end{figure*}

\begin{figure*}[ht]
    \centering
    \includegraphics[width=0.65\linewidth]{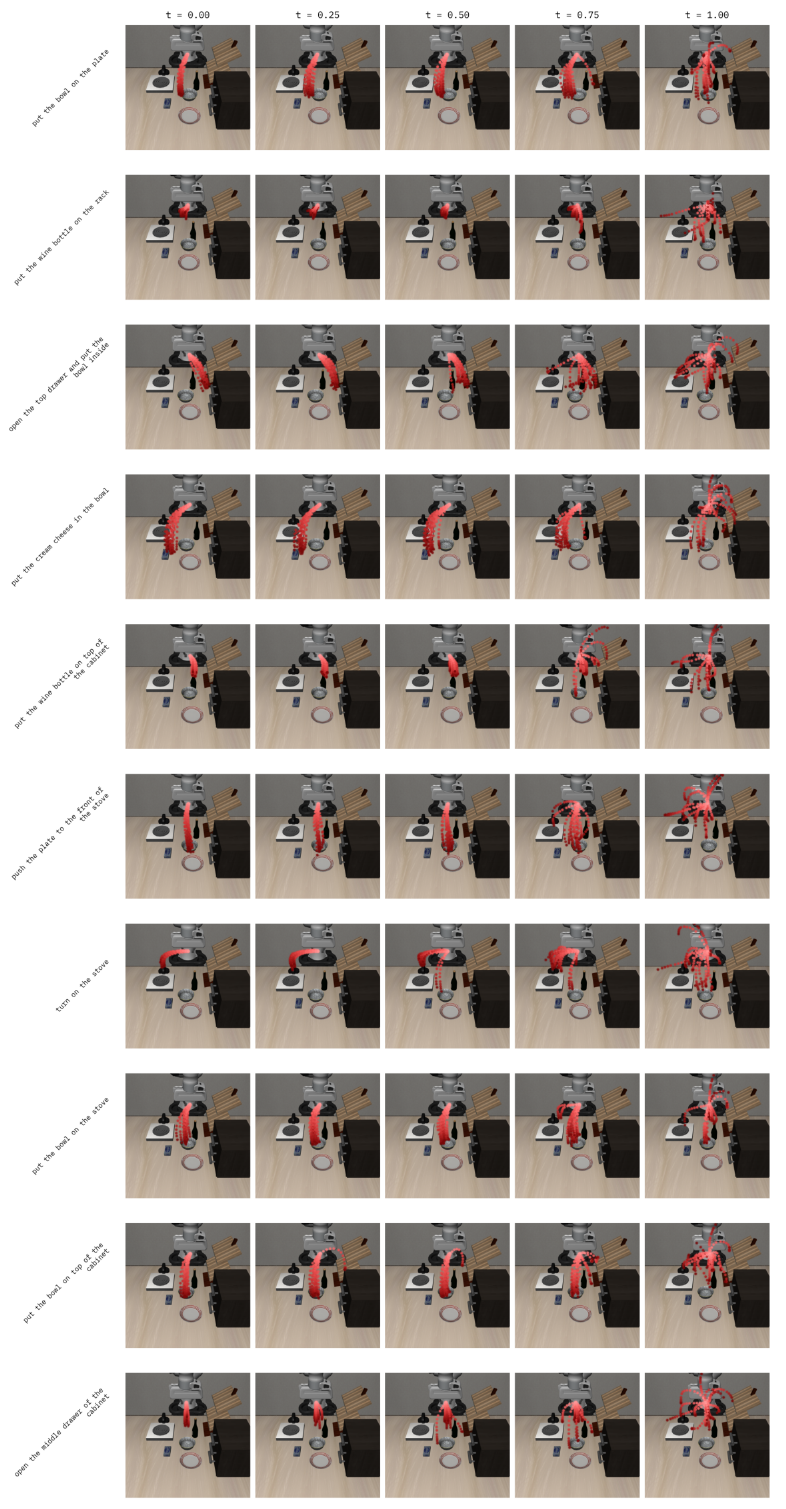}
    \caption{\textbf{\texttt{Libero-Goal} action distributions under context smoothing.} Each row corresponds to a distinct task in the \texttt{Libero-Goal} suite. Columns show rollouts as the observation embedding is progressively smoothed from the original context (left) toward the marginal distribution (right), interpolating between the conditional and marginal policy.}
    \label{fig:libero_goal_smoothing}
\end{figure*}

\end{document}